\documentclass{article}
\usepackage{iclr2026_conference,times}

\usepackage{hyperref}
\hypersetup{
    colorlinks=true,
    urlcolor=[HTML]{3f8991},
    linkcolor=[HTML]{8a3044},
    citecolor=[HTML]{607011},
  }
\usepackage{url}

\usepackage{amsthm,amsfonts,amssymb}
\usepackage{mathtools}
\usepackage[linesnumbered,ruled]{algorithm2e}

\usepackage{tikz}
\usetikzlibrary{arrows.meta,graphs,positioning,matrix,backgrounds,calc, decorations.pathmorphing,patterns,decorations.pathreplacing,shapes.misc,shapes.geometric}
\definecolor{tblblue}{RGB}{78,121,167}
\definecolor{tblsky}{RGB}{160,203,232}
\definecolor{tblorange}{RGB}{242,142,43}
\definecolor{tblapricot}{RGB}{255,190,125}
\definecolor{tblforest}{RGB}{89,161,79}
\definecolor{tblapple}{RGB}{140,209,125}
\definecolor{tblmustard}{RGB}{182,153,45}
\definecolor{tblyellow}{RGB}{241,206,99}
\definecolor{tblteal}{RGB}{73,152,148}
\definecolor{tblweirdgreen}{RGB}{134,188,182}
\definecolor{tblred}{RGB}{225,87,89}
\definecolor{tblflamingo}{RGB}{255,157,154}
\definecolor{tblgray}{RGB}{121,112,110}
\definecolor{tbllightgray}{RGB}{186,176,172}
\definecolor{tblpink}{RGB}{211,114,149}
\definecolor{tblrose}{RGB}{250,191,210}
\definecolor{axisgray}{RGB}{39,39,42}

\newtheorem{theorem}{Theorem}[section]
\newtheorem{definition}[theorem]{Definition}
\newtheorem{lemma}[theorem]{Lemma}

{%
  \leavevmode\nobreak\par
  \begin{list}%
    {}%
    {%
      \def\labelstyle{\itshape}
      \setlength{\topsep}{0pt}%
      \settowidth{\labelwidth}{\labelstyle Input:}%
      \setlength{\leftmargin}{\labelwidth}%
      \addtolength{\leftmargin}{\labelsep}%
      \setlength{\itemsep}{0pt}%
      \setlength{\parsep}{0pt}%
    }%
    \def\input{
  \item[\labelstyle Input:]}%
  }%
  {%
  \end{list}%
}

\title{Embracing Discrete Search: A Reasonable \\ Approach to Causal Structure Learning}

\author{Marcel Wienöbst$^1$, Leonard Henckel$^2$, Sebastian Weichwald\thanks{Correspondence: \texttt{sweichwald@math.ku.dk}}\hspace{0.13cm}$^{,3}$ \\ [0.2em]
  $^1$University of Lübeck, $^2$University College Dublin, $^3$University of Copenhagen
}

\iclrfinalcopy 
\begin{document}

\maketitle

\begin{abstract}
We present FLOP (Fast Learning of Order and Parents), a score-based causal
discovery algorithm for linear models. It pairs fast parent selection with
iterative Cholesky-based score updates, cutting run-times over prior
algorithms. This makes it feasible to fully embrace discrete search, enabling
iterated local search with principled order initialization to find graphs with
scores at or close to the global optimum. The resulting structures are highly
accurate across benchmarks, with near-perfect recovery in standard settings.
This performance calls for revisiting discrete search over graphs as a
reasonable approach to causal discovery.
\end{abstract}

\section{Introduction}

Learning about the directed acyclic graph (DAG) underlying a system's
data-generating process from observational data under causal sufficiency is a
fundamental causal discovery task~\citep{pearl2009causality}. Score-based
algorithms address this task by assigning penalized likelihood scores to each
DAG and seeking graphs whose scores are optimal. Identifiability theory asks
when such score-optimal graphs identify the target graph (or its equivalence
class) in the infinite-sample limit, with various results under different
assumptions and scores~\citep{chickering2002optimal,nandy2018high}.

Exact algorithms, that are guaranteed to find a score-optimal graph, have
exponential run-time and are feasible up to roughly 30
variables~\citep{koivisto2004exact,silander2006simple}. For larger graphs,
local search must be employed, which evaluates neighbouring graphs to find
graphs with better scores; canonical moves for this hill climbing are single
edge insertions, deletions, or reversals~\citep{heckerman1995learning}. In the
sample limit, greedy discrete search with a neighbourhood notion that respects
score equivalence provably finds a graph with optimal
score~\citep{chickering2002optimal}. In finite samples, scores are inexact and
local search may get stuck in local optima or, as we demonstrate, even find
graphs with better scores than the true graph. Finite-sample performance is a
practical challenge, despite the mature identifiability theory and asymptotic
guarantees.

Continuous optimization methods have emerged as a popular alternative. For
example, NOTEARS encodes acyclicity as a smooth constraint and optimizes a
surrogate objective~\citep{zheng2018dags}, with many
follow-ups~\citep{bello2022dagma,rolland22a}. Their supposed advantages have
been questioned empirically and
conceptually~\citep{reisach2021beware,reisach2023sorting,ng2024structure}.
Further, NP-hardness results often cited to dismiss discrete search do not
apply to the commonly considered discovery settings: The standard hardness
constructions rely on data-generating processes that involve unobserved
variables and cannot be represented by a DAG over only the observed
variables~\citep{chickering1996learning,chickering2004large}. When the
distribution is representable by a sparse DAG, discrete procedures
asymptotically recover the target graph with polynomially many independence
tests or score
evaluations~\citep{claassen2013learning,chickering2015selective}.

One of the core issues of score-based methods in practice are
finite-sample induced local optima~\citep{nielsen2012local}. Hence, the best-performing
heuristics in benchmarks~\citep{rios2021benchpress} are either able to escape
local optima, for example through simulated
annealing~\citep{kuipers2022efficient}, or realize larger
neighborhoods~\citep{pisinger2018large}, such as recent order-based
methods~\citep{lam2022greedy,andrews2023fast} with effective reinsertion moves
rather than only swapping neighboring
nodes~\citep{teyssier2012ordering,scanagatta2015learning}. This helps explain
the strong performance of the order-based BOSS
algorithm~\citep{andrews2023fast} and more recent order-based local
searches~\citep{pmlr-v267-li25bx} on common causal discovery benchmarks.
Continuous relaxations also alter the search space traversal, yet they have not
matched this performance and introduce additional challenges, for example
optimization complexity, convergence issues, edge thresholding, and having to
resort to surrogate objectives.

\begin{figure}
  \centering
  \begin{tikzpicture}[scale=0.548]
    \input{img/default_shd.tikz}
    \node[font=\small] at (0.55928969005743869, 1.30058963282937365) {$\text{FLOP}_0$};
    \node[font=\small] at (2.339726785980697, 0.68) {$\text{FLOP}_{20}$};
    \node[font=\small] at (5.252739895528542, 0.6582937365010799) {$\text{FLOP}_{100}$};
    \node[font=\small] at (7.16520515335865, 0.99) {BOSS};
    \node[font=\small] at (3.381505664393822, 2.1546652267818578) {PC};
    \node[font=\small] at (6.289976764476, 5.043952483801296) {GES};
    \node[font=\small] at (5.977417823969677, 2.220043196544277) {DAGMA};
  \end{tikzpicture}
  \begin{tikzpicture}[scale=0.548]
    \input{img/default_aid.tikz}
    \node[font=\small] at (0.55928969005743869, 1.80058963282937365) {$\text{FLOP}_0$};
    \node[font=\small] at (2.539726785980697, 1.0) {$\text{FLOP}_{20}$};
    \node[font=\small] at (5.022739895528542, 0.5082937365010799) {$\text{FLOP}_{100}$};
    \node[font=\small] at (6.16520515335865, 1.4) {BOSS};
    \node[font=\small] at (3.381505664393822, 3.7746652267818578) {PC};
    \node[font=\small] at (5.589976764476, 5.043952483801296) {GES};
    \node[font=\small] at (7.977417823969677, 6.860043196544277) {DAGMA};
  \end{tikzpicture}
  \caption{Run-time plotted against Structural Hamming Distance (left) and Ancestor Adjustment Identification Distance~\citep{henckeladjustment} (right) between the CPDAGs learned on linear ANM data and the target CPDAG corresponding to the underlying Erdős-Renyi generated DAG with 50 nodes, average degree 8 and 1000 samples drawn. Every point corresponds to one of 50 random instances; diamonds indicate averages. FLOP variants differ in the number of ILS restarts to escape local optima. The fraction of instances with exact CPDAG recovery is 40\% for BOSS and $\text{FLOP}_0$ and 60\% for $\text{FLOP}_{20}$ and $\text{FLOP}_{100}$, and zero for the remaining algorithms.}
  \label{figure:default:plot}
\end{figure}

\paragraph{Contributions.} We introduce FLOP (Fast Learning of Order and
Parents), a score-based structure learning algorithm for linear additive noise
models that fully embraces discrete search, and offer a Rust implementation at \href{https://github.com/CausalDisco/flopsearch}{github.com/CausalDisco/flopsearch} ready-to-use from Python via \texttt{pip install flopsearch}. The FLOP
algorithm adopts reinsertion- and order-based exploration of
DAGs~\citep{andrews2023fast} and adds four components that enable aggressive
search for graphs with optimal BIC score. First, we simplify parent selection
by re-initializing from the parent sets learned for the previous order, which
reduces compute and memory cost without degrading performance
(Section~\ref{sec:parentselection}). Second, we accelerate score computations
for the linear Gaussian BIC via efficient iterative updates of Cholesky
factorizations, which amortize cost across local moves
(Section~\ref{sec:cholesky}). Third, we develop a principled initialization
that, compared with a random initial order, reduces local parent selection
failures on ancestor-descendant pairs that are far apart and only weakly
dependent (Section~\ref{sec:initialorder}). Fourth, the computational gains
allow us to employ an iterated local search (ILS) metaheuristic to escape local optima (Section~\ref{sec:ils}). On standard
benchmarks, the order-based methods BOSS and FLOP achieve strong accuracy at
favorable run-time (see Figure~\ref{figure:default:plot}). FLOP's run-time
advantage can be translated into higher accuracy by extending the ILS budget,
with $\text{FLOP}_k$ denoting $k$ iterations. By treating compute budget as a
hyperparameter, our work highlights the link between run-time and finite‑sample
accuracy in causal discovery.

\section{Preliminaries}
We consider the problem of learning about the acyclic graph structure
underlying linear additive noise models (ANMs) from observational data. For a
causal DAG $G=(V,E)$ with node set $V = \{ 1, \dots, p\}$ and edge set $E
\subsetneq \{ i \rightarrow j \; \mid \; i,j \in V \}$, a linear ANM is defined
by a weight matrix $W \in \mathbb{R}^{p \times p}$ with $W_{i,j} \neq 0 \iff
i\to j \in E$ and a vector $N = [N_1,...,N_p]^T$ of jointly independent,
real-valued, zero-mean noise variables with finite fourth moment; the observed
variables $X = [X_1,...,X_p]^T$ are then defined by $X = W^T X + N$.

We denote the parents of $v\in V$, that is all $u\in V$ such that $u
\rightarrow v \in E$, by $\text{Pa}(v)$.
Every DAG $G=(V,E)$ can be associated to at least one linear order $\tau$, also
called \emph{topological} order, of the nodes such that $u \rightarrow v \in E$
implies $u$ coming before $v$ in $\tau$. A DAG is called Markovian to a
probability distribution if every variable is independent of its
non-descendants (all nodes not reachable from it with a directed path) given
its parents. We denote conditional independence by $\perp$.

As score to optimize, we choose the Bayesian Information
Criterion~\citep{schwarz1978estimating}, BIC for short, which is the common
choice in score-based structure learning for ANMs and is asymptotically
consistent under the faithfulness assumption~\citep{koller2009probabilistic}, which we employ throughout this paper. For a DAG
$G=(V,E)$ and a data set $D$ containing $n$ observations, it is defined as $k
\cdot \ln(n) - 2 \ln(\hat{L})$ with $k$ being the number of parameters and
$\hat{L}$ the maximized likelihood for the given DAG and data. The score can be
decomposed into local scores $\text{BIC}_D(G, X) = \sum_{v \in V} \ell_D(X_v,
X_{\text{Pa}(v)})$ and, for linear models with Gaussian noise, each local score
is given by $\ell_D(X_v, X_{\text{Pa}(v)}) = n \log(\widehat{\text{Var}}_D(X_v \mid
X_{\text{Pa}(v)})) + \lambda \ln(n)\, |\text{Pa}(v)|$ with $\lambda$ being a
penalty parameter.\footnote{Here, one needs to assume that the empirical covariance matrix is non‑degenerate and numerically well‑conditioned, excluding cases such as zero‑variance noise, high‑dimensional settings with $n < p$, or variance blow‑up severe enough to cause numerical instability.} The BIC is score-equivalent: Its value is the same for
Markov equivalent DAGs~\citep{verma1990equivalence} that imply the same
conditional independencies. In fact, with observational data under causal
sufficiency and without additional assumptions, only such an equivalence class
of DAGs is identifiable. Throughout, our target object is therefore the
equivalence class of the underlying DAG, represented by a completed partially
directed acyclic graph~\citep{andersson1997characterization}. Methods that
internally optimize over DAGs, such as continuous relaxations, or rely on extra
assumptions to identify a unique DAG, such as non-Gaussian linear
models~\citep{shimizu2011directlingam}, or noise-variance
conditions~\citep{park2020identifiability}, need to be evaluated via the
corresponding CPDAG for a fair comparison; evaluating a single DAG as if
identified is arbitrary under our assumptions and can be misleading.

We employ principles from the BOSS algorithm~\citep{andrews2023fast} to
optimize the BIC score over DAGs: We traverse the space of topological orders
of DAGs, iteratively moving to orders that result in better scoring DAGs when
selecting parents accordingly. Candidate orders are generated by taking a
variable and reinserting it at another position. Given an order $\tau$, we use
the grow-shrink procedure~\citep{margaritis2003learning} to construct a parent
set for each variable $v$ from its prefix, the variables preceding it in
$\tau$, and score the resulting DAG. Algorithm~\ref{algorithm:boss} shows the
BOSS reinsertion strategy, which we build on for the improved search in FLOP.

\begin{algorithm}[t]
  \DontPrintSemicolon
  \SetKwInOut{Input}{input}\SetKwInOut{Output}{output}
  \Input{Data set $D$ over $p$ variables.}
  \Output{A CPDAG $G$.}
  \SetKwFunction{FReinsert}{reinsert}
  \SetKwFunction{FFitParents}{growShrink}
  \SetKwFunction{FSum}{sum}
  \SetKwProg{Fn}{function}{}{end}

  \BlankLine

  $\tau \coloneq$ initial order of $\{1, \dots, p\}$ \tcp{Below, $\tau_i$ refers to the $i$-th element of $\tau$}
  \tcc{$P_v$, $\ell_v$ contain the parents and the local score of node $v$}
  \lForEach{$i \in \{1, \dots, p\}$}{
    $(P_{\tau_i}, \ell_{\tau_i}) \coloneq$ \FFitParents($\tau_i$, $\tau_{1:i-1}$, $D$)
  }\label{line:boss:growshrink}

  \BlankLine

  \Repeat{\FSum{$\ell$} $\geq$ \FSum{$\ell_{\text{old}}$}}{
    $(\tau^{\text{old}}, \ell^{\text{old}}) \coloneq (\tau, \ell)$ \;
    \For{$v \in \tau^{\text{old}}$}{
      $(\tau, \; P, \; \ell) \coloneq$ \FReinsert{$\tau$, $P$, $\ell$, $v$, $D$} \tcp{Optimal reinsertion for $v$.}
    }
  }

  \BlankLine

  \Return CPDAG of the DAG defined by parent sets $P_1, \dots, P_p$ \;
  \caption{Reinsertion-based local search as proposed by~\cite{andrews2023fast}.}
  \label{algorithm:boss}
\end{algorithm}

\begin{algorithm}[t]
  \DontPrintSemicolon
  \SetKwFunction{FReinsert}{reinsert}
  \SetKwFunction{FFitParents}{growShrink}
  \SetKwFunction{FSum}{sum}
  \SetKwProg{Fn}{function}{}{end}

  \Fn(\tcp*[h]{$P$ stores parents, $\ell$ local scores.}){\FReinsert{$\tau$, $P$, $\ell$, $v$, $D$}}{
    $i \coloneq$ position of $v$ in $\tau$ \;
    $(\hat{\tau}, \hat{P}, \hat{\ell}) \coloneq (\tau, P, \ell)$ \; 
    \ForEach(\tcp*[h]{Test reinsertions at later positions.}){$j \in \{i + 1, \dots, p\}$}{
      $(\tau_{j-1}, \tau_j) \coloneq (\tau_j, \tau_{j-1})$ \tcp*{Swap $\tau_{j-1}$ one position to the right.} 
      \tcp{Compute parents for changed prefixes of $\tau_{j-1}$ and $\tau_j$.}
      $(\ell_{\tau_j}, \; P_{\tau_j}) \coloneq$
      \FFitParents{$\tau_j$, $\tau_{1:j-1}$, $D$, $P_{\tau_j}$, $l_{\tau_j}$,
      $+\tau_{j-1}$} \; 
      $(\ell_{\tau_{j-1}}, \; P_{\tau_{j-1}}) \coloneq$
      \FFitParents{$\tau_{j-1}$, $\tau_{1:j-2}$, $D$, $P_{\tau_{j-1}}$,
      $l_{\tau_{j-1}}$, $-\tau_{j}$} \;
      \lIf{\FSum{$\ell$} $<$ \FSum{$\hat{\ell}$}}{
        $(\hat{\tau}, \hat{P}, \hat{\ell}) \coloneq (\tau, P, \ell)$
      }
    }
    \lForEach(\tcp*[h]{Analogous for earlier positions.}){$j \in \{i - 1, \dots, 1\}$}{\dots}}
    \Return $\hat{\tau}$, $\hat{P}$, $\hat{\ell}$ \;
  \caption{Find the best-scoring reinsertion of node $v$ in order $\tau$ given data $D$.}
  \label{algorithm:reinsert}
\end{algorithm}

\section{Scaling Up Order-Based Search} 
This section presents two speedups for order-based local search with
reinsertion moves, which yield significantly faster run-times than the
grow-shrink trees used in BOSS. FLOP still uses grow-shrink to obtain DAGs from
orders, but in a way that exploits the local, iterative search moves and
scoring.

\subsection{Starting Grow-Shrink from the Previous Parent Set}\label{sec:parentselection}
During the scoring of node-reinsertions, each node's candidate parent set, that is, the nodes coming before it in the order, changes by at most one node being 
inserted or deleted from its prefix.
Consider Algorithm~\ref{algorithm:reinsert} that finds the
best reinsertion for node $v$ currently at position $i$ in order $\tau$.
The possible
reinsertions of $v$ can be efficiently evaluated by
performing a sweep from position $i$ to the right (and also to the left; analogous code omitted),
moving it to
position $i+1$, $i+2$, and so on, by swapping it rightward.
At each step, the prefix of node $v$ increases by exactly one element,
while the prefixes of nodes originally
at positions $i+1$, $i+2$, and so on, lose exactly one node, namely $v$. 

Instead of running grow-shrink from the empty set at every step as in
BOSS, FLOP initializes grow-shrink with the previous parent set, that is, it
continues from the result of grow-shrink for the previous prefix, now with one
additional or one fewer node. The idea behind this strategy is that the parent
set typically changes little when the prefix changes by just one node, and so
this warm start makes parent selection far cheaper. Our implementation is given
in Algorithm~\ref{algorithm:growshrink}. Moreover, our grow-shrink does not
insist on inserting or removing the single best parent with largest
score improvement, making it \emph{non-greedy}; it adds or removes any parent that improves the score, even
if not maximally so. This eliminates the need for complicated grow-shrink tree
caching as used in BOSS. 

We show that the modified grow-shrink with warm start learns the restricted
Markov boundary of a node $v$ with respect to a set $Z$~\citep{lam2022greedy}.
This yields theoretical guarantees that the DAG learned by FLOP is the sparsest
Markovian one for the considered order~\citep{raskutti2018learning}.

\begin{definition}
  Let $P$ be a distribution over $X_1, \dots, X_p$. The \emph{restricted Markov
  boundary} of $X_v$ relative to a set $Z \subseteq \{X_1, \dots, X_p\}\setminus\{X_v\}$,
  denoted by $M(v, Z)$, is defined as a set of nodes $M \subseteq Z$ such that
    a) $X_v \perp (Z \setminus M) \mid M$
    and b)
    there exists no $M'\subset M$ such that $X_v \perp (Z \setminus M') \mid M'$.
\end{definition}

Under mild assumptions, the Markov boundary is unique~\citep{verma1990causal}.
As in GRaSP~\citep{lam2022greedy} and BOSS~\citep{andrews2023fast}, we learn it
using BIC score improvements in place of conditional independence
tests~\citep{margaritis2003learning}. In the sample limit, the local BIC score
$\ell(X_v, X_{\text{Pa}(v)} \cup \{X_u\})$ is smaller than $\ell(X_v,
X_{\text{Pa}(v)})$ if, and only if, $X_v \not\perp X_u \mid
X_{\text{Pa}(v)}$~\citep{koller2009probabilistic}. We show that this asymptotic
guarantee also carries over to the modified grow-shrink algorithm that starts
from an arbitrary initial parent set instead of the empty set.

\begin{lemma}\label{lemma:modified:gs}
  Let data set $D$ consist of $n$ i.i.d.\ observations of a probability
  distribution represented by a Bayesian network over variables $X_1,
  \dots, X_p$. Then, in the large sample limit of $n$, grow-shrink finds the
  restricted Markov boundary of node $v$ relative to a set $Z \subseteq \{X_1,
  \dots, X_p\} \setminus \{X_v\}$ when started with any initial set $P \subseteq Z$. 
\end{lemma}

\begin{proof}
  This follows directly from the proof of correctness of the grow-shrink
  algorithm in~\citep{margaritis2003learning} and its generalization to
  restricted Markov boundaries in~\citep{lam2022greedy}. Assume that at the end
  of the grow-phase, the current set of parents is $P_{\text{grow}}$. Thus, it
  holds that $X_v \perp X_u \mid P_{\text{grow}}$ for all $X_u \in Z \setminus
  P_{\text{grow}}$, or rephrased $X_v \perp Z \setminus P_{\text{grow}} \mid
  P_{\text{grow}}$. However, this would violate the uniqueness of the Markov
  boundary $M(v, Z)$ if it is not a subset of $P_{\text{grow}}$. This argument
  does not depend on the initial set $P$. The correctness of the shrink-phase
  is unchanged, too. 
\end{proof}

\begin{algorithm}[t]
  \DontPrintSemicolon
  \SetKwFunction{FGrowShrink}{growShrink}
  \SetKwFunction{FGrow}{grow}
  \SetKwFunction{FShrink}{shrink}
  \SetKwFunction{FLocalScore}{localScore}
  \SetKwProg{Fn}{function}{}{end}

  \Fn{\FGrowShrink{$u$, $Z$, $D$, $P_{\text{prev}}$, $\ell_{\text{prev}}$, $\delta$}}{
    \uIf(\tcp*[h]{$\delta$ is the node added to the candidate parents $Z$}){$\delta$ $ > 0$}{
      $P_{\text{new}} \coloneq P_{\text{prev}} \cup \{\delta\}$ \;
      $\ell_{\text{new}} \coloneq$ \FLocalScore{$u$, $P_{\text{new}}$, $D$} \tcp{score with $\delta$ added to parents}
      \lIf(\tcp*[h]{return if no improvement}){$\ell_{\text{new}} < \ell_{\text{prev}}$}{\Return $\ell_{\text{prev}}$, $P_{\text{prev}}$}\label{line:gs:break:add}
    }
    \uElseIf(\tcp*[h]{$|\delta|$ is the node removed from the candidates $Z$}){$\delta < 0$}{
      \lIf(\tcp*[h]{return if $|\delta|$ was no parent}){$|\delta| \not\in P_{\text{prev}}$}{\Return $\ell_{\text{prev}}$, $P_{\text{prev}}$}\label{line:gs:break:rem}
      $P_{\text{new}} \coloneq P_{\text{prev}} \setminus \{|\delta|\}$ \;
      $\ell_{\text{new}} \coloneq$ \FLocalScore{$u$, $P_{\text{new}}$, $D$} \;
    }
    \Else(\tcp*[h]{If no $\delta$ is provided, run from scratch.}){$(P_{\text{new}}, \ell_{\text{new}}) \coloneq (\emptyset,$ \FLocalScore{$u$, $\emptyset$, $D$} $)$ }
    \FGrow{$u$, $P_{\text{new}}$, $\ell_{\text{new}}$, $Z$, $D$} \;
    \FShrink{$u$, $P_{\text{new}}$, $\ell_{\text{new}}$, $Z$, $D$}

  }

  \BlankLine

  \Fn{\FGrow{$u$, $P$, $\ell$, $Z$, $D$}}{
    \Repeat{$P$ is unchanged}{\ForEach{$v \in Z \setminus P$}{
        \lIf{$\ell_{\text{new}} \coloneq$ \FLocalScore{$u$, $P \cup \{v\}$, $D$} $ < \ell$}{ $(\ell, P) \coloneq (\ell_{\text{new}}, P \cup \{v\})$}
    }}
  }
  
  \BlankLine
  
  \lFn(\tcp*[h]{Analogous to grow, thus omitted.}){\FShrink{$u$, $P$, $\ell$, $Z$, $D$}}{
    \dots 
  }
  \caption{Non-greedy grow-shrink with the option to start from a previous parent set $P_{\text{prev}}$.}
  \label{algorithm:growshrink}
\end{algorithm}

In addition to the warm start, we implement another optimization.
We pass the node $v$ that we are either inserting to (coded as $+v$) or removing from (coded as $-v$) the prefixes into grow-shrink.
If $v$ has
been removed and was not part of $P_{\text{prev}}$, we immediately return
$P_{\text{prev}}$. If node $v$ has been inserted to the prefix and does
not increase the score when added to $P_{\text{prev}}$, we again immediately
return $P_{\text{prev}}$. We show that these modifications preserve the guarantees above.
In the sample limit, FLOP returns a Markovian DAG, that is, one that induces no additional conditional independencies. 

\begin{theorem}\label{theorem:markovian}
  Let data set $D$ consist of $n$ i.i.d.\ observations of a probability distribution represented by a Bayesian network over $X_1, \dots, X_p$. In the sample limit of $n$, the CPDAG returned by FLOP is Markovian to $P$.
\end{theorem}

\begin{proof}
  This statement holds assuming that the grow-shrink procedure in FLOP finds a
  Markovian graph for each scored order. As the grow-shrink routines depend on
  the previous runs, we prove this by induction. Initially, a standard
  grow-shrink is run for the starting order (line~\ref{line:boss:growshrink} of
  Algorithm~\ref{algorithm:boss}), which yields parent sets corresponding to
  its sparsest Markovian
  DAG~\citep{raskutti2018learning,lam2022greedy}. Assume that the
  parent sets for the previous order have this property. By
  Lemma~\ref{lemma:modified:gs}, the modified grow-shrink, if run fully, finds
  the restricted Markov boundary with respect to the prefix and thus yields
  parent sets of the sparsest Markovian DAG. It remains to show that the two
  early breaks in lines~\ref{line:gs:break:add} and~\ref{line:gs:break:rem} of
  Algorithm~\ref{algorithm:growshrink} are correct, where the grow-shrink is
  not run. 

  If removed node $\delta$ was not part of the previous Markov boundary,
  clearly the Markov boundary remains unchanged for the reduced prefix, as
  neither the grow nor the shrink phase would add or remove a node. If added
  node $\delta$ does not increase the score for the enlarged prefix, this is the case, too. Thus, by Lemma~\ref{lemma:modified:gs}, the DAG learned by FLOP is Markovian and the statement follows by the fact that the CPDAG of such a DAG is returned.
\end{proof}

We remark that further modifications to FLOP in the subsequent sections do not
change this result. As with BOSS, one can make FLOP asymptotically consistent, provably yielding the true graph in the
sample limit under the faithfulness assumption, by running the backwards phase of
GES~\citep{chickering2002optimal} after termination of the local search.
However, we refrain from this, since FLOP already reaches state-of-the-art
finite-sample performance without it. More generally, any score-based discovery
algorithm can be made asymptotically consistent by running another consistent
algorithm, for example GES or PC, in parallel and returning whichever graph
attains the better score.

\subsection{Dynamic Cholesky Updates}\label{sec:cholesky}
We exploit the local structure of grow-shrink
to avoid recomputing the local score from scratch at every step.
Instead, we update the score from the previous parent set
which at each step changes only by a single
insertion or removal.
This idea is generic and applies to any local search that adds or removes one edge at a time in a local search, not only to order-based methods.

In the multivariate normal setting,
the local BIC score at node $u$ is
$\ell(X_u, X_{\text{Pa}(u)}) = n \log(\widehat{\text{Var}}_D(X_u \mid
X_{\text{Pa}(u)})) + \lambda \ln(n) |\text{Pa}(u)|$. Since the penalty depends
only on the size of the parent set, the work in computing or updating the local
score is in the likelihood term, in the estimated conditional variance of
$X_u$. A direct way to compute this would be to invert a submatrix of the
sample covariance matrix, but it is numerically more stable and faster to avoid
the matrix inversion in favor of using Cholesky factorizations. As shown in
Appendix~\ref{appendix:cholesky:covmat}, the bottom right entry of the Cholesky
factor of the covariance submatrix corresponding to $X_{\text{Pa}(u)}$ and
$X_u$ yields the square root of the conditional variance of $X_u$ given
$X_{\text{Pa}(u)}$. 

Computing the Cholesky decomposition of a $k \times k$ matrix requires
$(1/3)k^3$ floating-point operations. However, as discussed above, the
submatrix which we Cholesky-factorize changes by adding or removing only a
single row and column (corresponding to the added or removed parent node). We
therefore update the Cholesky factor instead of recomputing it, using standard
rank-one update and downdate routines~\citep{gill1974methods,golub2013matrix}.
These updates require $O(k^2)$ floating-point operations, shaving off a factor
$k$ compared to a fresh Cholesky decomposition. This run-time improvement
proves advantageous for larger and denser graphs. These Cholesky updates are
applicable to other score-based causal discovery algorithms, for example GES or
other hill-climbers. To our knowledge, this speedup has not been described in
prior causal discovery work.

\begin{figure}
  \centering
  \begin{tikzpicture}[xscale=0.548, yscale=0.46]
    \input{img/runtime.tikz}
    \node[font=\small] at (9.2528969005743869, 4.50058963282937365) {pre-FLOP (grow-shrink from empty set)};
    \node[font=\small] at (17, 6.48) {pre-FLOP (grow-shrink from previous parents)};
    \node[font=\small, align=center] at (19.46520515335865, 2.45) {FLOP \\ (grow-shrink from previous parents \\ and Cholesky updates)};
    \node[font=\small] at (3.281505664393822, 3.5046652267818578) {BOSS};
  \end{tikzpicture}
  \caption{Run-time in seconds, averaged over 50 repetitions with standard-deviation error bars, for ER graphs with average degree 16, 1000 samples, and $\{50, 100, 150, \dots, 500\}$ nodes.}
  \label{figure:runtime:plot}
\end{figure}

\subsection{Run-time Comparison}

In Figure~\ref{figure:runtime:plot}, we
compare the run-time of FLOP, which includes the two run-time
improvements described in this section, with two ablated versions,
termed \emph{pre-FLOP} in the plot, the first one using neither
optimization (thus only differing from BOSS by using \emph{non-greedy} grow-shrink) and the second one only using the grow-shrink started
at the previous parent set. For reference, we provide the run-time
of the BOSS implementation in the \texttt{Tetrad} software package.
The comparison with Tetrad is not apples-to-apples, since Tetrad is
written in Java and multithreaded, while our code is Rust and
single-threaded. The benchmark uses Erdős-Renyi graphs with average
degree 16, oriented according to a uniformly-random linear order
(details on the simulation setup are provided in
Section~\ref{section:simulations}). Each run has a 30 minute time
limit.

Both the modified grow-shrink and the Cholesky updates yield substantial
run-time reductions. With both optimizations, FLOP is more than a factor 100
faster than BOSS for graphs with 100 nodes and scales to 500 nodes, whereas
BOSS reaches the time limit for instances with 150 nodes. We note that accuracy
of the discovered graphs is similar on these instances for both methods, both
giving good, but not perfect results. In the following, we use the optimized
FLOP and build on these speedups to further improve the quality of the found
graphs.

\section{Improving the Accuracy of Order-Based Search}
This section presents two techniques to improve search accuracy. First, we
replace random initial orders with a principled initial order construction 
putting strongly-correlated nodes next to each other, which is critical on
directed paths in finite samples. Second, we use Iterated Local Search (ILS),
which perturbs a found solution and restarts the local search, trying to escape
local optima through additional compute. With these techniques, FLOP attains
state-of-the-art accuracy in simulations.

\subsection{Initial Order}\label{sec:initialorder}

Path graphs
$x_1 \rightarrow x_2 \rightarrow x_3 \rightarrow \dots \rightarrow x_p$
are challenging instances for order-based methods, which, to
our knowledge, have not been previously discussed in this context before.
On the left of
Figure~\ref{figure:accuracy:plot}, we compare different
algorithms on path graphs with 50 nodes.
FLOP with a random initial order
($\text{FLOP}_0^\text{rand}$), and BOSS are in fact the
worst-performing of all methods. 
A reason for this are far-apart ancestor-descendant pairs with very weak marginal dependence, for which the grow-shrink
procedure may fail to add edges, resulting in non-Markovian DAGs in finite
samples.
For example, if $x_i$ and $x_j$ with $i \ll j$  appear first in the order,
grow-shrink should, irrespective of the remaining order, make $x_i$ a parent of
$x_j$ for it to yield a Markovian DAG since $x_i$ and $x_j$ are marginally
dependent. However, the dependence between $x_i$ and $x_j$ may be too small for
grow-shrink to pick up on in finite samples.

As a remedy, we build the initial order so that strongly correlated nodes are
adjacent, facilitating grow-shrink to find a Markovian graph. To build the
order, we start with the two most correlated nodes and append, at each step, the
variable that can be best explained by variables already placed in the order,
that is, the one with the smallest residual variance when regressed onto the
nodes in the order. We standardize the data beforehand to avoid scale
artefacts. We compute this order efficiently, by iteratively constructing a
Cholesky decomposition of the covariance matrix choosing the next node in the
order according to their residual variance (see
Appendix~\ref{appendix:cholesky:covmat}). On 50-node paths
(Figure~\ref{figure:accuracy:plot}, left), FLOP with this initial order has an
average SHD on-par with PC and GES, the best performing algorithms on these
instances.

\begin{figure}
  \centering
  \begin{tikzpicture}[scale=0.548]
    \input{img/chain_shd.tikz}
    \node[font=\small] at (0.7528969005743869, 5.00058963282937365) {$\text{FLOP}_0^{\text{rand}}$};
    \node[font=\small] at (0.559726785980697, 1.58) {$\text{FLOP}_{0}$};
    \node[font=\small] at (4.26520515335865, 4.1) {BOSS};
    \node[font=\small] at (2.281505664393822, 0.5046652267818578) {PC};
    \node[font=\small] at (4.189976764476, 0.943952483801296) {GES};
    \node[font=\small] at (5.977417823969677, 2.220043196544277) {DAGMA};
  \end{tikzpicture}
  \begin{tikzpicture}[scale=0.548]
    \input{img/dense_shd.tikz}
    \node[font=\small] at (0.55928969005743869, 7.30058963282937365) {$\text{FLOP}_0$};
    \node[font=\small] at (0.699726785980697, 2.3) {$\text{FLOP}_{20}$};
    \node[font=\small] at (1.322739895528542, 1.1082937365010799) {$\text{FLOP}_{100}$};
    \node[font=\small] at (3.222739895528542, 0.0082937365010799) {$\text{FLOP}_{500}$};
    \node[font=\small] at (6.622739895528542, 0.0082937365010799) {Exact};
    \node[font=\small] at (1.71520515335865, 5.1) {$\text{BOSS}_0$};
    \node[font=\small] at (3.60520515335865, 3.05) {$\text{BOSS}_{20}$};
    \node[font=\small] at (7.30520515335865, 2.05) {$\text{BOSS}_{100}$};
  \end{tikzpicture}
  \caption{Run-time plotted against SHD on paths with 50 nodes for 1000 samples (left) and ER graphs with 25 nodes and average degree 16 for 50,000 samples (right). For the path graph, $\text{FLOP}_0$ finds the target graph in 72\% of instances, PC in 32\%, GES in 66\% and the remaining algorithms in none; for the ER graphs, $\text{FLOP}_{20}$ does so in 26\% of cases, $\text{FLOP}_{100}$ in 50\%, $\text{FLOP}_{500}$ in 56\%, Exact in 58\%, $\text{BOSS}_{100}$ in 4\% and the remaining algorithms in none.}
  \label{figure:accuracy:plot}
\end{figure}

\subsection{Iterated Local Search}\label{sec:ils}
Iterated Local Search (ILS) is a classic metaheuristic in discrete
optimization~\citep{lourencco2018iterated} that has been used in previous score-based search over DAGs and CPDAGs~\citep{LiuMetaheuristic2023,nazaret2025extremely}. It is a generic strategy that
combines local search with perturbations to escape local optima: Run local
search to a local optimum, perturb the best solution seen so far, then rerun
local search starting from this perturbation; repeat. In principle, this
procedure can be repeated indefinitely. 

For FLOP, the first local search starts from the initial order constructed as
described in the previous section. After that, the starting order for the
next local search is obtained by perturbing the best-found order by
$k$ random swaps of two (not necessarily adjacent) elements.
The idea being, that
orders near local optima are better
starting points than fully random ones.
We set $k = \ln p$ by default,
which we found to yield robust results balancing moving far enough to escape while staying in a promising basin.

On dense Erdős-Renyi graphs with 25 nodes and an average degree of 16
(Figure~\ref{figure:accuracy:plot}, right), increasing the number
of restarts of the local search (zero restarts amount to one local search, $x$
restarts to $x$ perturbations and new local searches after that), consistently
improves FLOP's accuracy. With 500 restarts, FLOP matches the exact score-based
algorithm while having a faster run-time (the exact score-based algorithm
implements the method by~\cite{silander2006simple} and uses multithreading). We
compare against BOSS with full random restarts, that is $x$ restarts mean $x+1$
independent runs of BOSS and returning the best-found solution. This is
computationally heavy and yields substantially smaller gains than FLOP's ILS
restarts.

ILS is an integral part of the FLOP algorithm. When calling FLOP, the user needs to specify either the number of restarts of the local
search or a time limit, and the solver runs ILS until the budget is exhausted.
This emphasizes the trade-off between run-time and accuracy inherent to
score-based causal discovery, but effectively ignored by the structure learning
community with its focus on one-shot heuristic algorithms.

\section{Simulations}
\label{section:simulations}
We empirically compare FLOP to other causal discovery methods. For
Figure~\ref{figure:default:plot}, we generate Erdős-Renyi (ER) graphs with 50
nodes and average degree 8. We also consider scale-free (SF) graphs with
density parameter $k=4$, generated by starting with a star graph of $k+1$ nodes
and adding further nodes by preferential attachment to $k$ existing nodes, and
DAGs from the bnlearn repository~\citep{scutari2010learning}, such as the Alarm
network~\citep{beinlich1989alarm}. We orient all graphs according to linear
orders drawn uniformly at random. For each graph, we generate 1000 samples from
a linear additive noise model with Gaussian noise (with mean 0 and variance
uniformly drawn from $[0.5, 2.0]$) and edge coefficients drawn uniformly from
$[-1, -0.25] \cup [0.25, 1]$. Each setting is repeated for 50 random instances.

In addition to FLOP and BOSS, we run PC~\citep{spirtes2001causation} as a
classical constraint-based method, GES~\citep{chickering2002optimal} as a
traditional score-based algorithm, and DAGMA~\citep{bello2022dagma} as a
gradient-based continuous optimization method. For BOSS, PC, and GES, we rely
on the implementation in Tetrad~\citep{ramsey18tetrad} through
\texttt{causal-cmd} version 1.12.0, for DAGMA we use version 1.1.0 of the
authors' implementation. The algorithms are run on a machine with 256GB of RAM
and an AMD Ryzen Threadripper 3970 CPU with 32 cores. We make no restrictions
on the number of threads the implementations may use (FLOP only uses a single
thread, whereas the other algorithms exploit multithreading) and report the
wall-clock time of their execution. We use standard parameters in the
literature, setting $\lambda_{\text{BIC}}=2$ for the BIC-based algorithms \citep[for a motivation of a higher penalty parameter than prescribed by the standard BIC, see][]{foygel2010extended}, $\alpha=0.01$ for PC and $\lambda_{\text{DAGMA}}
= 0.02$. As metric of accuracy, we report the Structural Hamming Distance (SHD)
for CPDAGs, that is, the number of node pairs with differing edge relations in
the compared graphs. If a method, such as DAGMA, returns a DAG, we first
compute the corresponding CPDAG and compare this to the CPDAG of the true DAG
(as we generally consider assumptions where only the CPDAG is identifiable).
For some settings, we also report the Ancestor Adjustment Identification
Distance (AID), measuring the mistakes when using the learned instead of the
true CPDAG for the downstream task of causal effect
identification~\citep{henckeladjustment}. For the PC algorithm, which does not
always return a graph satisfying the invariants of CPDAGs, such as acyclicity,
we report the AID only on runs that produced a valid CPDAG.

\begin{figure}
  \centering
  \begin{tikzpicture}[scale=0.548]
    \input{img/sf_shd.tikz}
    \node[font=\small] at (0.659726785980697, 1.98) {$\text{FLOP}_{0}$};
    \node[font=\small] at (2.359726785980697, 0.37) {$\text{FLOP}_{20}$};
    \node[font=\small] at (6.359726785980697, 0.3) {$\text{FLOP}_{100}$};
    \node[font=\small] at (6.00520515335865, 1.4) {BOSS};
    \node[font=\small] at (3.781505664393822, 3.5046652267818578) {PC};
    \node[font=\small] at (6.189976764476, 5.943952483801296) {GES};
    \node[font=\small] at (7.277417823969677, 4.250043196544277) {DAGMA};
  \end{tikzpicture}
  \begin{tikzpicture}[scale=0.548]
    \input{img/alarm_shd.tikz}
    \node[font=\small] at (0.659726785980697, 3.08) {$\text{FLOP}_{0}$};
    \node[font=\small] at (0.659726785980697, 0.67) {$\text{FLOP}_{20}$};
    \node[font=\small] at (2.749726785980697, -0.07) {$\text{FLOP}_{100}$};
    \node[font=\small] at (2.749726785980697, 1.9) {BOSS};
    \node[font=\small] at (4.749726785980697, 1.9) {PC};
    \node[font=\small] at (4.849726785980697, 1.3) {GES};
    \node[font=\small] at (8.249726785980697, 4.7) {DAGMA};
  \end{tikzpicture}
  \caption{Run-time plotted against SHD on SF graphs (left) and the Alarm network, consisting of 37 nodes and 46 edges, (right), both for 1000 samples. For the SF graphs, $\text{FLOP}_{20}$ finds the target CPDAG in 6\% of cases, $\text{FLOP}_{100}$ in 10\%, the remaining algorithms in none; for the Alarm network, $\text{FLOP}_{0}$ does so in 2\% of cases, $\text{FLOP}_{20}$ in 74\%, $\text{FLOP}_{100}$ in 82\%, BOSS in 6\%, GES in 16\%, DAGMA and PC in none.}
  \label{figure:sf:alarm}
\end{figure}

Figure~\ref{figure:default:plot} shows run-time versus SHD (lower left is better). On SF graphs, the order-based algorithms
clearly outperform PC, GES, and DAGMA; FLOP with ILS improves further.
Even with 100 ILS restarts, FLOP's run time is comparable to BOSS.
On the Alarm
network instances, the improvements through ILS are even more apparent, and
with it FLOP obtains near-perfect results.

\begin{figure}
  \begin{tikzpicture}[scale=0.54]
    \input{img/default_bic.tikz}
    \node[font=\small] at (0.550928969005743869, 1.40058963282937365) {$\text{FLOP}_0$};
    \node[font=\small] at (2.029726785980697, 0.60) {$\text{FLOP}_{20}$};
    \node[font=\small] at (5.452739895528542, 0.5002937365010799) {$\text{FLOP}_{100}$};
    \node[font=\small] at (7.40520515335865, 0.60) {BOSS};
    \node[font=\small] at (5.30520515335865, 4.50) {GES};
  \end{tikzpicture}
  \begin{tikzpicture}[scale=0.54]
    \input{img/sf_bic.tikz}
    \node[font=\small] at (0.550928969005743869, 2.10058963282937365) {$\text{FLOP}_0$};
    \node[font=\small] at (2.029726785980697, 0.60) {$\text{FLOP}_{20}$};
    \node[font=\small] at (3.352739895528542, -0.1002937365010799) {$\text{FLOP}_{100}$};
    \node[font=\small] at (6.10520515335865, 0.50) {BOSS};
    \node[font=\small] at (6.00520515335865, 5.00) {GES};
  \end{tikzpicture}
  \caption{Run-time plotted against the BIC difference to the ground-truth graph for ER graphs on the left and SF graphs on the right. For the ER graphs, $\text{FLOP}_{0}$ finds a graph with better or equal BIC score than the true graph in 48\% of cases, $\text{FLOP}_{20}$ and $\text{FLOP}_{100}$ in 84\% of cases, BOSS in 52\% of cases and GES in 0\% cases. For the SF graphs, $\text{FLOP}_{0}$ finds such a graph in 6\% of cases, $\text{FLOP}_{20}$ in 76\% of cases, $\text{FLOP}_{100}$ in 94\% of cases, BOSS in 6\% of cases and GES again in 0\% of cases.}
  \label{figure:bic:ersf}
\end{figure}

Graphs returned by FLOP achieve a lower SHD than competing score-based methods
due to better optimization of the BIC score. This is shown in
Figure~\ref{figure:bic:ersf}, where we report the BIC score difference to the
ground-truth DAG. Generally, the results look qualitatively similar to the SHD
plots for the presented settings. However, for the SF graphs, it can be seen
that the BIC, e.g., for $\text{FLOP}_{20}$ is close to zero, whereas the
SHD for many instances lies clearly above zero. In fact, for a majority of
runs, the BIC score of the graph found by $\text{FLOP}_{20}$ is even (slightly)
better than the BIC score of the ground-truth graph showing that the global BIC
optimum does not identify the ground truth in these cases.

We also evaluated the DAGMA loss function with MLE parameters fitted to
the graph returned by FLOP and observed this to produce a lower loss compared
to the graph and parameters returned by DAGMA itself. This casts doubt on the
idea that gradient-based methods relying on differentiable DAG-constraints have
an inherent advantage in optimizing their target score compared to discrete
search. While these methods may offer other benefits, our results suggest that
those likely come from aspects other than optimization quality.

\section{Discussion}\label{sec:discussion}
In score-based causal discovery, two questions arise: (1) Is the true graph
score-optimal? and (2) Can we find a score-optimal graph? Here we introduce
FLOP, an efficient discrete optimization algorithm to search for graphs
minimizing the Gaussian BIC score. Indeed, when the true graph of a linear ANM
has the globally optimal Gaussian BIC, FLOP typically recovers this graph or
outputs one very close to it. This is further supported by simulations in
Appendix~\ref{appendix:further:benchmarks} including uniform noise,
unstandardized data, data based on an adaptation of the Onion
method~\citep{andrews2024better}, and real-world networks from
bnlearn~\citep{scutari2010learning}. Across these settings, FLOP attains
state-of-the-art accuracy, typically achieving better BIC and lower SHD in a
fraction of the run-time of competing methods. When assumptions are violated
(Appendix~\ref{subsection:nonlinear}-\ref{subsection:sachs}) or sample sizes
are too small for asymptotic guarantees to hold
(Appendix~\ref{subsection:dense} and Figure~\ref{figure:pathfinder} in
Appendix~\ref{subsection:bnlearn}), FLOP still optimizes the Gaussian BIC as
intended and finds graphs with better BIC score than the ground truth, but
graph recovery suffers because the scoring criterion does not identify the true
graph.

These results highlight that it is reasonable to revisit and embrace discrete
search for causal structure learning. FLOP often finds graphs that are
score-optimal or score better than the target graph and makes the link between
accuracy and speed explicit: A computationally efficient search permits more exploration through iterated local search and that leads to better-scoring graphs. Our findings also
recalibrate what is considered hard. First, ER graphs with 50 nodes and about
200 edges are often presented as challenging, yet for linear ANMs order-based
discrete search solves them reliably and quickly. 
Second, on widely used linear benchmarks, the discrete optimization of the Gaussian BIC is feasible, not a bottleneck, and often more efficient and reliable than other optimization methods. Instead, key challenges in causal discovery lie in
designing and selecting appropriate scoring criteria that identify the true
graph as score-optimal not only asymptotically but with high probability also
on finite samples.

At the same time, advancing causal discovery in practice remains difficult even
on small graphs, since the ground truth is rarely known and assumptions are
violated. It has been feasible for decades to find a global BIC optimum with
exact exponential-time search up to roughly 30
variables~\citep{koivisto2004exact,silander2006simple}. FLOP extends strong BIC
optimization to substantially larger graphs, but that does not make the
practical problems go away. Our work shifts the attention away from inflated
combinatorial hardness rhetoric and from a misattributed gap between
asymptotic theory and observed finite-sample performance, toward the immense
challenges causal discovery faces outside of synthetic benchmarks~\citep{reisach2021beware,gobler2024texttt,mogensen2024causal,brouillard2025landscape,gamella2025causal,gururaghavendran2025can,jorgensen2025causal}. 

\subsubsection*{Acknowledgments}
Marcel Wienöbst thanks Kenneth Langedal for fruitful discussions and introducing him to iterated local search. Sebastian Weichwald was supported by a research grant (0069071) from Novo Nordisk Fonden. This research was supported by the Pioneer Centre for AI, DNRF grant number P1.

\bibliography{main}

\begin{thebibliography}{57}
\providecommand{\natexlab}[1]{#1}
\providecommand{\url}[1]{\texttt{#1}}
\expandafter\ifx\csname urlstyle\endcsname\relax
  \providecommand{\doi}[1]{doi: #1}\else
  \providecommand{\doi}{doi: \begingroup \urlstyle{rm}\Url}\fi

\bibitem[Andersson et~al.(1997)Andersson, Madigan, and
  Perlman]{andersson1997characterization}
Steen~A. Andersson, David Madigan, and Michael~D. Perlman.
\newblock A characterization of {Markov} equivalence classes for acyclic
  digraphs.
\newblock \emph{The Annals of Statistics}, 25\penalty0 (2):\penalty0 505--541,
  1997.

\bibitem[Andrews \& Kummerfeld(2024)Andrews and Kummerfeld]{andrews2024better}
Bryan Andrews and Erich Kummerfeld.
\newblock Better simulations for validating causal discovery with the
  {DAG}-adaptation of the onion method.
\newblock \emph{arXiv preprint arXiv:2405.13100}, 2024.

\bibitem[Andrews et~al.(2023)Andrews, Ramsey, Sanchez~Romero, Camchong, and
  Kummerfeld]{andrews2023fast}
Bryan Andrews, Joseph Ramsey, Ruben Sanchez~Romero, Jazmin Camchong, and Erich
  Kummerfeld.
\newblock Fast scalable and accurate discovery of {DAG}s using the best order
  score search and grow shrink trees.
\newblock \emph{Advances in Neural Information Processing Systems}, pp.\
  63945--63956, 2023.

\bibitem[Beinlich et~al.(1989)Beinlich, Suermondt, Chavez, and
  Cooper]{beinlich1989alarm}
Ingo~A. Beinlich, Henri~Jacques Suermondt, R.~Martin Chavez, and Gregory~F.
  Cooper.
\newblock The {ALARM} monitoring system: {A} case study with two probabilistic
  inference techniques for belief networks.
\newblock In \emph{European Conference on Artificial Intelligence in Medicine},
  pp.\  247--256, 1989.

\bibitem[Bello et~al.(2022)Bello, Aragam, and Ravikumar]{bello2022dagma}
Kevin Bello, Bryon Aragam, and Pradeep Ravikumar.
\newblock {DAGMA}: Learning {DAG}s via {M}-matrices and a log-determinant
  acyclicity characterization.
\newblock \emph{Advances in Neural Information Processing Systems}, pp.\
  8226--8239, 2022.

\bibitem[Brouillard et~al.(2025)Brouillard, Squires, Wahl, Sachs, Drouin, and
  Sridhar]{brouillard2025landscape}
Philippe Brouillard, Chandler Squires, Jonas Wahl, Karen Sachs, Alexandre
  Drouin, and Dhanya Sridhar.
\newblock The landscape of causal discovery data: Grounding causal discovery in
  real-world applications.
\newblock In \emph{Causal Learning and Reasoning}, pp.\  834--873, 2025.

\bibitem[Chickering(1996)]{chickering1996learning}
David~Maxwell Chickering.
\newblock Learning {B}ayesian networks is {NP}-complete.
\newblock In \emph{Learning from Data: Artificial Intelligence and Statistics
  V}, pp.\  121--130. Springer, 1996.

\bibitem[Chickering(2002)]{chickering2002optimal}
David~Maxwell Chickering.
\newblock Optimal structure identification with greedy search.
\newblock \emph{Journal of Machine Learning Research}, 3:\penalty0 507--554,
  2002.

\bibitem[Chickering \& Meek(2015)Chickering and Meek]{chickering2015selective}
David~Maxwell Chickering and Christopher Meek.
\newblock Selective greedy equivalence search: {F}inding optimal {B}ayesian
  networks using a polynomial number of score evaluations.
\newblock In \emph{Uncertainty in Artificial Intelligence}, pp.\  211--219,
  2015.

\bibitem[Chickering et~al.(2004)Chickering, Heckerman, and
  Meek]{chickering2004large}
David~Maxwell Chickering, David Heckerman, and Chris Meek.
\newblock Large-sample learning of {B}ayesian networks is {NP}-hard.
\newblock \emph{Journal of Machine Learning Research}, 5:\penalty0 1287--1330,
  2004.

\bibitem[Claassen et~al.(2013)Claassen, Mooij, and
  Heskes]{claassen2013learning}
Tom Claassen, Joris~M. Mooij, and Tom Heskes.
\newblock Learning sparse causal models is not {NP}-hard.
\newblock In \emph{Uncertainty in Artificial Intelligence}, pp.\  172–181,
  2013.

\bibitem[Foygel \& Drton(2010)Foygel and Drton]{foygel2010extended}
Rina Foygel and Mathias Drton.
\newblock Extended {Bayesian} information criteria for {Gaussian} graphical
  models.
\newblock In \emph{Advances in Neural Information Processing Systems}, pp.\
  604--612, 2010.

\bibitem[Gamella et~al.(2025)Gamella, Peters, and
  B{\"u}hlmann]{gamella2025causal}
Juan~L Gamella, Jonas Peters, and Peter B{\"u}hlmann.
\newblock Causal chambers as a real-world physical testbed for {AI}
  methodology.
\newblock \emph{Nature Machine Intelligence}, 7\penalty0 (1):\penalty0
  107--118, 2025.

\bibitem[Gill et~al.(1974)Gill, Golub, Murray, and Saunders]{gill1974methods}
Philip~E. Gill, Gene~H. Golub, Walter Murray, and Michael~A. Saunders.
\newblock Methods for modifying matrix factorizations.
\newblock \emph{Mathematics of Computation}, 28\penalty0 (126):\penalty0
  505--535, 1974.

\bibitem[G{\"o}bler et~al.(2024)G{\"o}bler, Windisch, Drton, Pychynski, Roth,
  and Sonntag]{gobler2024texttt}
Konstantin G{\"o}bler, Tobias Windisch, Mathias Drton, Tim Pychynski, Martin
  Roth, and Steffen Sonntag.
\newblock {causalAssembly}: Generating realistic production data for
  benchmarking causal discovery.
\newblock In \emph{Causal Learning and Reasoning}, pp.\  609--642, 2024.

\bibitem[Golub \& Van~Loan(2013)Golub and Van~Loan]{golub2013matrix}
Gene~H. Golub and Charles~F. Van~Loan.
\newblock \emph{Matrix Computations}.
\newblock JHU press, 2013.

\bibitem[Gururaghavendran \& Murray(2025)Gururaghavendran and
  Murray]{gururaghavendran2025can}
Rajesh Gururaghavendran and Eleanor~J. Murray.
\newblock Can algorithms replace expert knowledge for causal inference? {A}
  case study on novice use of causal discovery.
\newblock \emph{American Journal of Epidemiology}, 194\penalty0 (5):\penalty0
  1399--1409, 2025.

\bibitem[Heckerman et~al.(1995)Heckerman, Geiger, and
  Chickering]{heckerman1995learning}
David Heckerman, Dan Geiger, and David~M. Chickering.
\newblock Learning {B}ayesian networks: The combination of knowledge and
  statistical data.
\newblock \emph{Machine Learning}, 20\penalty0 (3):\penalty0 197--243, 1995.

\bibitem[Heckerman et~al.(1992)Heckerman, Horvitz, and
  Nathwani]{heckerman1992toward}
David~Earl Heckerman, Eric~J. Horvitz, and Bharat~N. Nathwani.
\newblock Toward normative expert systems: Part i the pathfinder project.
\newblock \emph{Methods of Information in Medicine}, 31\penalty0 (02):\penalty0
  90--105, 1992.

\bibitem[Henckel et~al.(2024)Henckel, W{\"u}rtzen, and
  Weichwald]{henckeladjustment}
Leonard Henckel, Theo W{\"u}rtzen, and Sebastian Weichwald.
\newblock Adjustment identification distance: A gadjid for causal structure
  learning.
\newblock In \emph{Uncertainty in Artificial Intelligence}, pp.\  1569--1598,
  2024.

\bibitem[Jensen \& Jensen(1996)Jensen and Jensen]{jensenmildew}
Allan~Leck Jensen and Finn~Verner Jensen.
\newblock {MIDAS}: {A}n influence diagram for management of mildew in winter
  wheat.
\newblock In \emph{Uncertainty in Artificial Intelligence}, pp.\  349--356,
  1996.

\bibitem[Jørgensen et~al.(2025)Jørgensen, Gresele, and
  Weichwald]{jorgensen2025causal}
Frederik~Hytting Jørgensen, Luigi Gresele, and Sebastian Weichwald.
\newblock {What is causal about causal models and representations?}
\newblock \emph{{arXiv preprint arXiv:2501.19335}}, 2025.

\bibitem[Koivisto \& Sood(2004)Koivisto and Sood]{koivisto2004exact}
Mikko Koivisto and Kismat Sood.
\newblock Exact {B}ayesian structure discovery in {B}ayesian networks.
\newblock \emph{Journal of Machine Learning Research}, 5:\penalty0 549--573,
  2004.

\bibitem[Koller \& Friedman(2009)Koller and Friedman]{koller2009probabilistic}
Daphne Koller and Nir Friedman.
\newblock \emph{Probabilistic Graphical Models: Principles and Techniques}.
\newblock MIT Press, 2009.

\bibitem[Kuipers et~al.(2022)Kuipers, Suter, and Moffa]{kuipers2022efficient}
Jack Kuipers, Polina Suter, and Giusi Moffa.
\newblock Efficient sampling and structure learning of {B}ayesian networks.
\newblock \emph{Journal of Computational and Graphical Statistics}, 31\penalty0
  (3):\penalty0 639--650, 2022.

\bibitem[Lam et~al.(2022)Lam, Andrews, and Ramsey]{lam2022greedy}
Wai-Yin Lam, Bryan Andrews, and Joseph Ramsey.
\newblock Greedy relaxations of the sparsest permutation algorithm.
\newblock In \emph{Uncertainty in Artificial Intelligence}, pp.\  1052--1062,
  2022.

\bibitem[Li et~al.(2025)Li, Qian, Wang, Li, Zhang, and Zhou]{pmlr-v267-li25bx}
Mingjia Li, Hong Qian, Tian-Zuo Wang, Shujun Li, Min Zhang, and Aimin Zhou.
\newblock Strong and weak identifiability of optimization-based causal
  discovery in non-linear additive noise models.
\newblock In \emph{International Conference on Machine Learning}, pp.\
  35753--35768, 2025.

\bibitem[Liu et~al.(2023)Liu, Gao, Wang, Ru, and Zhang]{LiuMetaheuristic2023}
Xiaohan Liu, Xiaoguang Gao, Zidong Wang, Xinxin Ru, and Qingfu Zhang.
\newblock A metaheuristic causal discovery method in directed acyclic graphs
  space.
\newblock \emph{Knowledge-Based Systems}, 276:\penalty0 110749, 2023.

\bibitem[Louren{\c{c}}o et~al.(2018)Louren{\c{c}}o, Martin, and
  St{\"u}tzle]{lourencco2018iterated}
Helena~Ramalhinho Louren{\c{c}}o, Olivier~C. Martin, and Thomas St{\"u}tzle.
\newblock Iterated local search: Framework and applications.
\newblock In \emph{Handbook of Metaheuristics}, pp.\  129--168. Springer, 2018.

\bibitem[Margaritis(2003)]{margaritis2003learning}
Dimitris Margaritis.
\newblock \emph{Learning {B}ayesian network model structure from data}.
\newblock PhD thesis, School of Computer Science, Carnegie Mellon University,
  2003.

\bibitem[Mogensen et~al.(2024)Mogensen, Rathsman, and
  Nilsson]{mogensen2024causal}
S{\o}ren~Wengel Mogensen, Karin Rathsman, and Per Nilsson.
\newblock Causal discovery in a complex industrial system: A time series
  benchmark.
\newblock In \emph{Causal Learning and Reasoning}, pp.\  1218--1236, 2024.

\bibitem[Nandy et~al.(2018)Nandy, Hauser, and Maathuis]{nandy2018high}
Preetam Nandy, Alain Hauser, and Marloes~H. Maathuis.
\newblock High-dimensional consistency in score-based and hybrid structure
  learning.
\newblock \emph{The Annals of Statistics}, 46\penalty0 (6A):\penalty0
  3151--3183, 2018.

\bibitem[Nazaret \& Blei(2024)Nazaret and Blei]{nazaret2025extremely}
Achille Nazaret and David Blei.
\newblock Extremely greedy equivalence search.
\newblock In \emph{Uncertainty in Artificial Intelligence}, pp.\  2716--2745,
  2024.

\bibitem[Ng et~al.(2020)Ng, Ghassami, and Zhang]{ng2020role}
Ignavier Ng, AmirEmad Ghassami, and Kun Zhang.
\newblock On the role of sparsity and {DAG} constraints for learning linear
  {DAGs}.
\newblock \emph{Advances in Neural Information Processing Systems},
  33:\penalty0 17943--17954, 2020.

\bibitem[Ng et~al.(2024)Ng, Huang, and Zhang]{ng2024structure}
Ignavier Ng, Biwei Huang, and Kun Zhang.
\newblock Structure learning with continuous optimization: A sober look and
  beyond.
\newblock In \emph{Causal Learning and Reasoning}, pp.\  71--105, 2024.

\bibitem[Nielsen et~al.(2003)Nielsen, Kocka, and Pe{\~{n}}a]{nielsen2012local}
Jens~Dalgaard Nielsen, Tom{\'{a}}s Kocka, and Jos{\'{e}}~M. Pe{\~{n}}a.
\newblock On local optima in learning {B}ayesian networks.
\newblock In \emph{Uncertainty in Artificial Intelligence}, pp.\  435--442,
  2003.

\bibitem[Park(2020)]{park2020identifiability}
Gunwoong Park.
\newblock Identifiability of additive noise models using conditional variances.
\newblock \emph{Journal of Machine Learning Research}, 21\penalty0
  (75):\penalty0 1--34, 2020.

\bibitem[Pearl(2009)]{pearl2009causality}
Judea Pearl.
\newblock \emph{Causality}.
\newblock Cambridge University Press, second edition, 2009.

\bibitem[Peters \& B{\"u}hlmann(2014)Peters and
  B{\"u}hlmann]{peters2014identifiability}
Jonas Peters and Peter B{\"u}hlmann.
\newblock Identifiability of {Gaussian} structural equation models with equal
  error variances.
\newblock \emph{Biometrika}, 101\penalty0 (1):\penalty0 219--228, 2014.

\bibitem[Pisinger \& Ropke(2018)Pisinger and Ropke]{pisinger2018large}
David Pisinger and Stefan Ropke.
\newblock Large neighborhood search.
\newblock In \emph{Handbook of Metaheuristics}, pp.\  99--127. Springer, 2018.

\bibitem[Ramsey et~al.(2018)Ramsey, Zhang, Glymour, Romero, Huang,
  Ebert-Uphoff, Samarasinghe, Barnes, and Glymour]{ramsey18tetrad}
Joseph~D. Ramsey, Kun Zhang, Madelyn Glymour, Ruben~Sanchez Romero, Biwei
  Huang, Imme Ebert-Uphoff, Savini Samarasinghe, Elizabeth~A. Barnes, and Clark
  Glymour.
\newblock {TETRAD} - {A} toolbox for causal discovery.
\newblock In \emph{International Workshop on Climate Informatics}, pp.\
  89--92, 2018.

\bibitem[Raskutti \& Uhler(2018)Raskutti and Uhler]{raskutti2018learning}
Garvesh Raskutti and Caroline Uhler.
\newblock Learning directed acyclic graph models based on sparsest
  permutations.
\newblock \emph{Stat}, 7\penalty0 (1):\penalty0 e183, 2018.

\bibitem[Reisach et~al.(2021)Reisach, Seiler, and Weichwald]{reisach2021beware}
Alexander Reisach, Christof Seiler, and Sebastian Weichwald.
\newblock Beware of the simulated {DAG}! {C}ausal discovery benchmarks may be
  easy to game.
\newblock \emph{Advances in Neural Information Processing Systems}, pp.\
  27772--27784, 2021.

\bibitem[Reisach et~al.(2023)Reisach, Tami, Seiler, Chambaz, and
  Weichwald]{reisach2023sorting}
Alexander Reisach, Myriam Tami, Christof Seiler, Antoine Chambaz, and Sebastian
  Weichwald.
\newblock A scale-invariant sorting criterion to find a causal order in
  additive noise models.
\newblock In \emph{Advances in Neural Information Processing Systems}, pp.\
  785--807, 2023.

\bibitem[Rios et~al.(2025)Rios, Moffa, and Kuipers]{rios2021benchpress}
Felix~L. Rios, Giusi Moffa, and Jack Kuipers.
\newblock Benchpress: A versatile platform for structure learning in causal and
  probabilistic graphical models.
\newblock \emph{Journal of Statistical Software}, 114:\penalty0 1--43, 2025.

\bibitem[Rolland et~al.(2022)Rolland, Cevher, Kleindessner, Russell, Janzing,
  Sch{\"o}lkopf, and Locatello]{rolland22a}
Paul Rolland, Volkan Cevher, Matth{\"a}us Kleindessner, Chris Russell, Dominik
  Janzing, Bernhard Sch{\"o}lkopf, and Francesco Locatello.
\newblock Score matching enables causal discovery of nonlinear additive noise
  models.
\newblock In \emph{International Conference on Machine Learning}, pp.\
  18741--18753, 2022.

\bibitem[Sachs et~al.(2005)Sachs, Perez, Pe'er, Lauffenburger, and
  Nolan]{sachs2005causal}
Karen Sachs, Omar Perez, Dana Pe'er, Douglas~A. Lauffenburger, and Garry~P.
  Nolan.
\newblock Causal protein-signaling networks derived from multiparameter
  single-cell data.
\newblock \emph{Science}, 308\penalty0 (5721):\penalty0 523--529, 2005.

\bibitem[Scanagatta et~al.(2015)Scanagatta, de~Campos, Corani, and
  Zaffalon]{scanagatta2015learning}
Mauro Scanagatta, Cassio~P. de~Campos, Giorgio Corani, and Marco Zaffalon.
\newblock Learning {B}ayesian networks with thousands of variables.
\newblock \emph{Advances in Neural Information Processing Systems}, pp.\
  1864--1872, 2015.

\bibitem[Schwarz(1978)]{schwarz1978estimating}
Gideon Schwarz.
\newblock Estimating the dimension of a model.
\newblock \emph{The Annals of Statistics}, pp.\  461--464, 1978.

\bibitem[Scutari(2010)]{scutari2010learning}
Marco Scutari.
\newblock Learning {B}ayesian networks with the bnlearn {R} package.
\newblock \emph{Journal of Statistical Software}, 35:\penalty0 1--22, 2010.

\bibitem[Shimizu et~al.(2011)Shimizu, Inazumi, Sogawa, Hyvarinen, Kawahara,
  Washio, Hoyer, and Bollen]{shimizu2011directlingam}
Shohei Shimizu, Takanori Inazumi, Yasuhiro Sogawa, Aapo Hyvarinen, Yoshinobu
  Kawahara, Takashi Washio, Patrik~O. Hoyer, and Kenneth Bollen.
\newblock {DirectLiNGAM}: {A} direct method for learning a linear
  non-{G}aussian structural equation model.
\newblock \emph{Journal of Machine Learning Research}, 12:\penalty0 1225--1248,
  2011.

\bibitem[Silander \& Myllym{\"a}ki(2006)Silander and
  Myllym{\"a}ki]{silander2006simple}
Tomi Silander and Petri Myllym{\"a}ki.
\newblock A simple approach for finding the globally optimal {B}ayesian network
  structure.
\newblock In \emph{Uncertainty in Artificial Intelligence}, pp.\  445--452,
  2006.

\bibitem[Spirtes et~al.(2000)Spirtes, Glymour, and
  Scheines]{spirtes2001causation}
Peter Spirtes, Clark Glymour, and Richard Scheines.
\newblock \emph{Causation, Prediction, and Search}.
\newblock {MIT} Press, second edition, 2000.

\bibitem[Teyssier \& Koller(2005)Teyssier and Koller]{teyssier2012ordering}
Marc Teyssier and Daphne Koller.
\newblock Ordering-based search: a simple and effective algorithm for learning
  {Bayesian} networks.
\newblock In \emph{Uncertainty in Artificial Intelligence}, pp.\  584--590,
  2005.

\bibitem[Verma \& Pearl(1988)Verma and Pearl]{verma1990causal}
Thomas Verma and Judea Pearl.
\newblock Causal networks: {S}emantics and expressiveness.
\newblock In \emph{Uncertainty in Artificial Intelligence}, pp.\  69--78, 1988.

\bibitem[Verma \& Pearl(1990)Verma and Pearl]{verma1990equivalence}
Thomas Verma and Judea Pearl.
\newblock Equivalence and synthesis of causal models.
\newblock In \emph{Uncertainty in Artificial Intelligence}, pp.\  255--270,
  1990.

\bibitem[Zheng et~al.(2018)Zheng, Aragam, Ravikumar, and Xing]{zheng2018dags}
Xun Zheng, Bryon Aragam, Pradeep~K. Ravikumar, and Eric~P. Xing.
\newblock {DAG}s with {NO TEARS}: Continuous optimization for structure
  learning.
\newblock \emph{Advances in Neural Information Processing Systems}, pp.\
  9492--9503, 2018.

\end{thebibliography}
\bibliographystyle{iclr2026_conference}

\newpage

\appendix

\section{High-Level Description of the FLOP Algorithm}
 As an overview of the FLOP algorithm, we provide Figure~\ref{figure:controlflow}, describing the high-level control flow of FLOP. The algorithm begins by computing an initial order for the given data set $D$. Afterwards, the ILS loop starts with a local search aiming to improve the order $\pi$ through reinsertions. Reinsertions are done by moving a node $v$ to its locally optimal position until no further improvements are possible (Algorithm~\ref{algorithm:boss}). This optimal reinsertion is computed as described in Algorithm~\ref{algorithm:reinsert}, which relies on the grow-shrink described in Algorithm~\ref{algorithm:growshrink} for updating the parents. This grow-shrink starts at the previous parent set and uses efficient Cholesky updates for scoring as described in Section~\ref{sec:cholesky}. 
 
 After the local search completes, the best found graph/order is updated ($\pi_{\text{best}}$ in Figure~\ref{figure:controlflow}) if the score is lower than the previous best. This best-scoring order found thus far is then perturbed as described in Section~\ref{sec:ils} and the perturbed copy is then used as the starting point for the next local search. We note that this procedure ensures that more ILS iterations can only improve the score of $\pi_{\text{best}}$ and thus of the returned graph $G(\pi_{\text{best}})$ because $\pi_{\text{best}}$ is only updated if the local search after the perturbation yields a better scoring order.
 
 We note that the reinsertion-based local search follows the general principle of the BOSS algorithm, which consists of reinserting nodes at their locally optimal position until no further (local) improvements are possible. However, FLOP and BOSS differ in the following key aspects:
 \begin{itemize}
   \item The grow-shrink procedure of BOSS starts with the empty set. Moreover, it is a greedy grow-shrink that always inserts or removes the node with the largest local score improvement. In contrast, FLOP accepts any improving insertion in the grow-, and any improving removal in the shrink-phase. For the implementation of grow-shrink BOSS relies at its core on an intricate data structure called grow-shrink trees, which FLOP avoids. Overall, this allows FLOP to obtain a better run-time performance compared to BOSS.
   \item FLOP uses Cholesky updates for efficient iterative scoring during the local search and, in particular, in the grow-shrink routine. This yields further run-time gains.
   \item FLOP makes use of an iterated local search (ILS) that allows spending more compute for improved BIC optimization. As more ILS restarts can never yield worse scoring graphs, this effectively trades off compute with accuracy. Due to FLOPs run-time improvements this yields a free lunch with regard to accuracy gains.
   \item BOSS starts the local search with a random order. This leads to performance deteriorations on path instances as shown in Section~\ref{sec:initialorder}. In contrast, FLOP explicitly constructs the initial order to avoid such problems.
 \end{itemize}
 
 \begin{figure}
   \centering
   \begin{tikzpicture}
     \node(init) [rectangle, draw, rounded corners=1mm, align=center] at (0,0) {Find initial order \\ cf. Section~\ref{sec:initialorder}};
     \node(ls) [rectangle, draw, rounded corners=1mm, align=center] at (4,0) {Local search \\ cf. Algorithms~\ref{algorithm:boss},~\ref{algorithm:reinsert},~\ref{algorithm:growshrink}};
     \node(update) [rectangle, draw, rounded corners=1mm, align=center] at (9,0) {Update $\pi_{\text{best}}$ to $\pi'$ \\ if $\text{BIC}(\pi') < \text{BIC}(\pi_{\text{best}})$};
     \node(perturb) [rectangle, draw, rounded corners=1mm, align=center] at (4,2.5) {Perturb $\pi_{\text{best}}$ \\ with random swaps};
 
     \node(restart) [align=center] at (7,1.25) {Repeat $k$ times (ILS) \\ cf. Section~\ref{sec:ils}};
 
     \draw[->, semithick, >={[round,sep]Stealth}] (0,1.5) -- (init) node[midway, right]{$D$};
     \draw[->, semithick, >={[round,sep]Stealth}] (init) -- (ls) node[midway,above]{$\pi$};
     \draw[->, semithick, >={[round,sep]Stealth}] (ls) -- (update) node[midway,above]{$\pi'$};
     \draw[->, semithick, >={[round,sep]Stealth}] (perturb) -- (ls) node[midway,right]{$\pi$};
     \draw[->, semithick, >={[round,sep]Stealth}] (update) -- (12.5,0) node[midway,above]{$G(\pi_{\text{best}})$};
     \draw[->, semithick, >={[round,sep]Stealth}] (update) -- (9,2.5) -- (perturb) node[midway,below]{$\pi_{\text{best}}$};
   \end{tikzpicture}
   \caption{Visualization of the general control flow in the FLOP algorithm.}
   \label{figure:controlflow}
 \end{figure}
 
 \section{Cholesky Decomposition of the Covariance Matrix}
 \label{appendix:cholesky:covmat}
 
 Let $X_1,\dots,X_p$ be real-valued centered random variables with finite second moments
 and with full-rank covariance matrix
 $\Sigma=\left[\operatorname{Cov}(X_r, X_c)\right]_{r,c\in \{1,..,p\}}\in\mathbb{R}^{p\times p}$,
 that is, for all $j\in\{1,...,p\}$, $\mathbb{E}(X_j)=0$ and $0 < \Sigma_{j,j}=\operatorname{Var}(X_j) < \infty$;
 further, $\Sigma$ is symmetric positive definite and admits a unique Cholesky factorization $\Sigma=LL^\top$
 with $L$ lower triangular and strictly positive diagonal.
 
 For $j\in\{1,...,p\}$,
 let $\widehat{X}_j$ denote the best linear predictor of $X_j$ from its predecessors $X_1,...,X_{j-1}$,
 that is, the ordinary least squares projection.
 
 Then for all $j\in\{1,...,p\}$,
 \begin{equation*}
 L_{jj}^2
 =
 \operatorname{Var}(X_j-\widehat{X}_j),
 \end{equation*}
 that is,
 $L_{jj}$ is the standard deviation of the least squares residuals
 when linearly regressing $X_j$ onto its predecessors.

 To obtain the statement, fix $j\in\{1,\dots,p\}$.
 Block-partition the leading $j\times j$ principal submatrices of $\Sigma$ and $L$:
 \[
 \Sigma_{1:j,\,1:j}
 =
 \begin{pmatrix}
 \overbrace{\Sigma_{1:(j-1),\,1:(j-1)}}^{\Sigma'}
  &
 \overbrace{\Sigma_{1:(j-1),\,j}}^{s}
  \\[1ex]
 \underbrace{\Sigma_{j,\,1:(j-1)}}_{s^\top}
  &
 \underbrace{\Sigma_{j,j}}_{c}
 \end{pmatrix},
 \qquad
 L_{1:j,\,1:j}
 =
 \begin{pmatrix}
 \overbrace{L_{1:(j-1),\,1:(j-1)}}^{L'}
  & 0
  \\[1ex]
 \underbrace{L_{j,\,1:(j-1)}}_{r}
  & \underbrace{L_{j,j}}_{\ell}
 \end{pmatrix}.
 \]
 
 From $\Sigma=LL^\top$ we get the block identities
 \[
 \Sigma' = L'L'^\top,
 \qquad
  s = L'r^\top,
 \qquad
 c = r r^\top + \ell^2.
 \]
 Since $L'$ is full rank, the second identity gives $r^\top=L'^{-1}s$, hence
 \[
 r r^\top
 = s^\top\!\left(L'^{-\top}L'^{-1}\right)s
 = s^\top\!\left(L'L'^{\top}\right)^{-1}s
 = s^\top\Sigma'^{-1}s.
 \]
 Substituting into $c = rr^\top+\ell^2$ yields
 \[
 \ell^2 \;=\; c - s^\top\Sigma'^{-1}s.
 \]
 
 On the other hand, the centered OLS problem
 \[
 \min_{a\in\mathbb{R}^{j-1}} \; \mathbb{E}\left[(X_j-a^\top X_{1:(j-1)})^2\right]
 \]
 has normal equations $\Sigma'a^\star=s$,
 so $a^\star=\Sigma'^{-1}s$ is the unique minimizer,
 $\widehat{X}_j = s^\top \Sigma'^{-1} X_{1:(j-1)}$,
 and the minimal mean squared error is
 \begin{align*}
 \operatorname{Var}(X_j-\widehat{X}_j)
 &=
 \operatorname{Var}(X_j) + \operatorname{Var}(\widehat{X}_j) - 2 \operatorname{Cov}(X_j, \widehat{X}_j) \\
 &= c
 \;\; + \;\;
 s^\top\Sigma'^{-1} \, \operatorname{Var}( X_{1:(j-1)}) \, \Sigma'^{-1} s
 \;\; - \;\;
 2 s^\top \Sigma'^{-1} \, \operatorname{Cov}(X_j, X_{1:(j-1)}) \\
 &= c \;\; + \;\; s^\top \Sigma'^{-1} \Sigma' \Sigma'^{-1} s \;\; - \;\; 2 s^\top \Sigma'^{-1} s \\
 &= c - s^\top \Sigma'^{-1} s
 \end{align*}
 This equals $\ell^2$, that is, $L_{jj}^2=\operatorname{Var}(X_j-\widehat{X}_j)$, as claimed.
 
 We remark that if $X_1,...,X_p$ are jointly Gaussian, then
 $\operatorname{Var}(X_j-\widehat{X}_j) = \operatorname{Var}(X_j\,|\,X_{1:(j-1)})$.
 
 \section{Further Benchmark Settings}
 \label{appendix:further:benchmarks}
 
 In this section, we consider further benchmark settings to investigate the
 stability of FLOP under different graph and data generation procedures. If not
 specified otherwise, we consider ER graphs with 50 nodes and average degree 8,
 and 1000 samples being drawn from the underlying linear additive noise model.
 We also consider settings where the assumptions of FLOP with a linear Gaussian
 BIC are (potentially) violated, such as uniform noise
 (Subsection~\ref{subsection:uniform}) and non-linear relations
 (Subsection~\ref{subsection:nonlinear}) as well as semi-synthetic
 (Subsection~\ref{subsection:causalAssembly}) and real-world data
 (Subsection~\ref{subsection:sachs}). In some settings, such as the
 uniform-noise case this leads to no performance degradation, whereas in others
 it leads to significantly higher SHDs compared to settings where the
 assumptions are satisfied. We observe that in these settings, FLOP still reliably
 optimizes the BIC, meaning the performance of FLOP in large parts depends on
 how well the scoring criterion is suited to the data. Hence, practitioners need
 to be careful that the assumptions hold when applying FLOP. Moreover, this
 shifts the focus in research from designing optimization algorithms towards the
 development of efficient and practical scoring criteria (see also the discussion in Section~\ref{sec:discussion}).
 
 \subsection{Uniform Noise} \label{subsection:uniform}
 To check the performance of FLOP (using the Gaussian
 BIC to learn the CPDAG underlying a linear ANM) under non-Gaussian noise, we
 generate data with noise sampled uniformly from $[-1, 1]$. As the plot on the
 left of Figure~\ref{figure:uniform:raw} shows, there is no performance
 degradation (of any algorithm). Moreover, we compared the methods to
 DirectLiNGAM~\citep{shimizu2011directlingam}, which is based on identifiability
 theory for non-Gaussian noise. DirectLinGAM gets low SHD on these instances,
 but in the 50 repetitions never recovered the ground truth.
 
 \subsection{Raw Data} \label{subsection:rawdata}
 To avoid varsortability of the
 instances~\citep{reisach2021beware}, we typically standardize the data in the
 benchmarks as mention in Section~\ref{section:simulations}. As an exception, we
 consider instances with unstandardized data on the right of
 Figure~\ref{figure:uniform:raw}. As expected, we find that DAGMA performs
 significantly better than in the standardized settings. The performance of the
 other algorithms does not vary significantly. We also note that in FLOP we
 choose to always standardize the data to obtain a scale-invariant algorithm.
 
 \begin{figure}
   \centering
   \begin{tikzpicture}[scale=0.548]
     \input{img/uniform_shd.tikz}
     \node[font=\small] at (0.50928969005743869, 1.10058963282937365) {$\text{FLOP}_0$};
     \node[font=\small] at (1.839726785980697, 0.38) {$\text{FLOP}_{20}$};
     \node[font=\small] at (4.852739895528542, -0.1082937365010799) {$\text{FLOP}_{100}$};
     \node[font=\small] at (3.30520515335865, 0.81) {BOSS};
     \node[font=\small] at (1.881505664393822, 2.0546652267818578) {PC};
     \node[font=\small] at (3.889976764476, 5.043952483801296) {GES};
     \node[font=\small] at (3.977417823969677, 2.320043196544277) {DAGMA};
     \node[font=\small] at (8.977417823969677, 0.760043196544277) {LiNGAM};
   \end{tikzpicture}
   \begin{tikzpicture}[scale=0.548]
     \input{img/raw_shd.tikz}
     \node[font=\small] at (0.45928969005743869, 1.10058963282937365) {$\text{FLOP}_0$};
     \node[font=\small] at (2.139726785980697, 0.52) {$\text{FLOP}_{20}$};
     \node[font=\small] at (3.852739895528542, -0.1582937365010799) {$\text{FLOP}_{100}$};
     \node[font=\small] at (5.86520515335865, 0.80) {BOSS};
     \node[font=\small] at (3.381505664393822, 3.0546652267818578) {PC};
     \node[font=\small] at (5.289976764476, 5.143952483801296) {GES};
     \node[font=\small] at (4.177417823969677, 1.520043196544277) {DAGMA};
   \end{tikzpicture}
   \caption{Run-time against SHD for data sampled with uniform instead of Gaussian noise on the left and for unstandardized data on the right (both settings are based on ER graphs with 50 nodes and average degree 8 with 1000 samples drawn). In the uniform noise case, $\text{FLOP}_0$ and BOSS find the target CPDAG in 34\% of cases, $\text{FLOP}_{20}$ and $\text{FLOP}_{100}$ in 54\% of the cases, the remaining algorithms in none. On unstandardized data, BOSS finds the target CPDAG in 22\% of cases, $\text{FLOP}_0$ in 34\% and $\text{FLOP}_{20}$ and $\text{FLOP}_{100}$ in 54\% of cases, the remaining algorithms in none.}
   \label{figure:uniform:raw}
 \end{figure}

 \begin{figure}
   \centering
   \begin{tikzpicture}[scale=0.548]
     \input{img/onion_shd.tikz}
     \node[font=\small] at (0.450928969005743869, 3.90058963282937365) {$\text{FLOP}_0$};
     \node[font=\small] at (1.039726785980697, 1.68) {$\text{FLOP}_{20}$};
     \node[font=\small] at (3.252739895528542, 1.6802937365010799) {$\text{FLOP}_{100}$};
     \node[font=\small] at (4.30520515335865, 3.21) {BOSS};
     \node[font=\small] at (6.181505664393822, 4.3546652267818578) {PC};
     \node[font=\small] at (2.889976764476, 4.143952483801296) {GES};
     \node[font=\small] at (7.477417823969677, 5.720043196544277) {DAGMA};
   \end{tikzpicture}
   \begin{tikzpicture}[scale=0.548]
     \input{img/dense_1000_shd.tikz}
     \node[font=\small] at (0.55928969005743869, 7.30058963282937365) {$\text{FLOP}_0$};
     \node[font=\small] at (-0.949726785980697, 5.8) {$\text{FLOP}_{20}$};
     \node[font=\small] at (0.682739895528542, 4.9682937365010799) {$\text{FLOP}_{100}$};
     \node[font=\small] at (1.822739895528542, 4.4082937365010799) {$\text{FLOP}_{500}$};
     \node[font=\small] at (7.122739895528542, 4.5582937365010799) {Exact};
     \node[font=\small] at (1.21520515335865, 6.5) {$\text{BOSS}_0$};
     \node[font=\small] at (3.40520515335865, 5.45) {$\text{BOSS}_{20}$};
     \node[font=\small] at (9.30520515335865, 6.25) {$\text{BOSS}_{100}$};
   \end{tikzpicture}
   \caption{Run-time against SHD for data sampled with the DAG-adaptation of the Onion method on the left (again on ER graphs with 50 nodes and average degree 8 with 1000 samples drawn) and for dense ER graphs (25 nodes, average degree 16) with 1000 samples generated in the standard way on the right. In both settings, none of the algorithms ever recover the target CPDAG.}
   \label{figure:onion:dense}
 \end{figure}
 
 \subsection{DaO Data} We also consider the DAG-adaptation of the Onion
 method~\citep{andrews2024better} as a way to sample data from an ANM. This
 method has been proposed to avoid artefacts in the data, such as
 R2-sortability~\citep{reisach2023sorting}, which could be inadvertently or
 explicitly exploited to game benchmarks.
 In line with the results by~\cite{andrews2024better}, we find that this
 sampling method yields harder-to-identify instances, with FLOP nonetheless
 performing best; however, all methods produce SHDs greater than 50, as shown
 in the left panel of Figure 8. This may be caused by weak causal relationships
 or (near)-faithfulness violations in the data. It is, however, not a failure
 in the optimization, as we observed that, for FLOP and other score-based
 algorithms, the BIC score of the learned graph was better than the one of the
 ground truth, suggesting non-identifiability of the true CPDAG under the BIC
 for the provided number of samples.
 
 \begin{figure}
   \centering
   \begin{tikzpicture}[scale=0.54]
     \input{img/large_accuracy_250.tikz}
     \node[font=\small] at (0.450928969005743869, 0.90058963282937365) {$\text{FLOP}_0$};
     \node[font=\small] at (2.329726785980697, 0.08) {$\text{FLOP}_{20}$};
     \node[font=\small] at (4.752739895528542, 0.7002937365010799) {$\text{FLOP}_{100}$};
     \node[font=\small] at (6.90520515335865, 0.70) {BOSS};
     \node[font=\small] at (0.381505664393822, 5.3546652267818578) {PC};
     \node[font=\small] at (2.189976764476, 3.443952483801296) {GES};
   \end{tikzpicture}
   \begin{tikzpicture}[scale=0.54]
     \input{img/large_accuracy_500.tikz}
     \node[font=\small] at (0.450928969005743869, 0.80058963282937365) {$\text{FLOP}_0$};
     \node[font=\small] at (2.939726785980697, 0.68) {$\text{FLOP}_{20}$};
     \node[font=\small] at (8.652739895528542, 0.6802937365010799) {$\text{FLOP}_{100}$};
     \node[font=\small] at (2.401505664393822, 6.8546652267818578) {PC};
     \node[font=\small] at (2.419976764476, 2.143952483801296) {GES};
   \end{tikzpicture}
   \caption{Run-time against SHD for ER graphs with 250 nodes on the left and with 500 nodes on the right (average degree 8 and 1000 samples drawn). BOSS times out on the latter instances. In both settings, none of the algorithms ever recover the target CPDAG.}
   \label{figure:er:large}
 \end{figure}

 \subsection{Dense ER Graphs} \label{subsection:dense}
 In the main paper, we considered dense ER graphs
 (25 nodes and average degree 16) in a setting with 50,000 samples. Due to the
 denseness of the graph such a large amount of samples is necessary to identify
 the target graph. Here, we show the performance of the algorithms for
 significantly fewer samples, namely 1000 samples, as in the other simulations.
 As can be seen on the right of Figure~\ref{figure:onion:dense}, FLOP still
 performs quite well, however, the algorithms are much closer with regard to the SHD. 
 
 We note that we again compared the BIC score of the graph returned by FLOP with
 the ground-truth graph as well as the other algorithms. We found that FLOP and the exact algorithm consistently found graphs with a better BIC score than the ground truth or the other discovery algorithms' output graphs. This illustrates that the sample size is too small for the BIC score to reliably identify the true CPDAG.
 
 Another thing to note is that compared to the setting with 50000 samples in the
 main text, both BOSS and FLOP run faster on instances with 1000 samples,
 whereas there is no noticeable difference for the exact algorithm. The reason
 for this increased run-time for larger samples sizes is that the BIC penalizes
 edges stronger for smaller sample sizes with the penalty term growing with $\ln
 n$ and the likelihood term proportional with $n$. Thus, intermediate graphs in
 the search are typically denser for high-sample settings, which increases the
 computational effort. 
 
 \begin{figure}
   \centering
   \begin{tikzpicture}[scale=0.54]
     \input{img/mildew_shd.tikz}
     \node[font=\small] at (0.450928969005743869, 3.80058963282937365) {$\text{FLOP}_0$};
     \node[font=\small] at (0.629726785980697, 0.55) {$\text{FLOP}_{20}$};
     \node[font=\small] at (3.352739895528542, -0.1002937365010799) {$\text{FLOP}_{100}$};
     \node[font=\small] at (5.30520515335865, 2.40) {BOSS};
     \node[font=\small] at (4.701505664393822, 1.2746652267818578) {PC};
     \node[font=\small] at (2.989976764476, 2.243952483801296) {GES};
     \node[font=\small] at (8.989976764476, 5.043952483801296) {DAGMA};
   \end{tikzpicture}
   \begin{tikzpicture}[scale=0.54]
     \input{img/barley_shd.tikz}
     \node[font=\small] at (0.450928969005743869, 3.50058963282937365) {$\text{FLOP}_0$};
     \node[font=\small] at (1.029726785980697, 0.75) {$\text{FLOP}_{20}$};
     \node[font=\small] at (2.852739895528542, -0.2002937365010799) {$\text{FLOP}_{100}$};
     \node[font=\small] at (5.20520515335865, 2.10) {BOSS};
     \node[font=\small] at (3.101505664393822, 2.0046652267818578) {PC};
     \node[font=\small] at (2.989976764476, 4.643952483801296) {GES};
     \node[font=\small] at (8.989976764476, 4.843952483801296) {DAGMA};
   \end{tikzpicture}
   \caption{Run-time against SHD for the Mildew network on the left~\citep{jensenmildew}, which consists of 35 nodes and 46 edges, and the Barley network on the right~\citep{scutari2010learning}, which consists of 48 nodes and 84 edges. For the Mildew network, BOSS finds the target CPDAG in 2\% of cases, $\text{FLOP}_{20}$ in 48\% and $\text{FLOP}_{100}$ in 52\% of the cases, the remaining algorithms in none. For the Barley network, GES finds the ground truth in 4\% of cases, $\text{FLOP}_0$ in 8\%, BOSS finds the target CPDAG in 12\%, $\text{FLOP}_{20}$ in 90\% and $\text{FLOP}_{100}$ in 94\% of cases, PC and DAGMA in none.}
   \label{figure:mildew:barley}
 \end{figure}
 
 \subsection{Large ER graphs} 
 We also report the accuracy for large ER graphs with 250 and 500 nodes and average degree 8 in Figure~\ref{figure:er:large}. Here, DAGMA does not terminate within the time limit for either instances and BOSS does not for the graphs with 500 nodes. Overall, similar accuracy results as before can be observed though notably PC appears to get worse with an increased number of variables in comparison with GES. 
 
 \begin{figure}
   \begin{tikzpicture}[scale=0.54]
     \input{img/pathfinder_shd.tikz}
     \node[font=\small] at (0.350928969005743869, 5.70058963282937365) {$\text{FLOP}_0$};
     \node[font=\small] at (2.029726785980697, 4.90) {$\text{FLOP}_{20}$};
     \node[font=\small] at (6.352739895528542, 5.6002937365010799) {$\text{FLOP}_{100}$};
     \node[font=\small] at (8.50520515335865, 4.60) {BOSS};
   \end{tikzpicture}
   \begin{tikzpicture}[scale=0.54]
     \input{img/pathfinder_bic.tikz}
     \node[font=\small] at (0.450928969005743869, -3.10058963282937365) {$\text{FLOP}_0$};
     \node[font=\small] at (1.629726785980697, -5.60) {$\text{FLOP}_{20}$};
     \node[font=\small] at (6.352739895528542, -5.8002937365010799) {$\text{FLOP}_{100}$};
     \node[font=\small] at (8.55520515335865, -3.09) {BOSS};
   \end{tikzpicture}
   \caption{Run-time against SHD for the Pathfinder network, which consists of 109 nodes and 195 edges, on the left. On the right, the BIC score difference compared to the ground-truth graph for the Pathfinder network. None of the algorithms ever recover the target CPDAG.}
   \label{figure:pathfinder}
 \end{figure}

 \subsection{bnlearn graphs} \label{subsection:bnlearn}
 In addition to the random graphs, we also consider real-world networks from the
 bnlearn repository, namely the Mildew~\citep{jensenmildew}, Barley and the
 Pathfinder network~\citep{heckerman1992toward}. All three are too large such
 that exact score-based algorithm based on dynamic programming could be used,
 with Mildew consisting of 35 nodes and 46 edges, Barley of 48 nodes and 84
 edges, and Pathfinder of 109 nodes and 195 arcs. In all cases, we generate the
 data synthetically in the same manner as before. For Mildew and Barley on the
 left and right of Figure~\ref{figure:mildew:barley}, FLOP performs
 significantly better than other methods and, in particular, that the ILS is
 needed to get close-to-perfect accuracy on these instances. For Pathfinder on
 the left of Figure~\ref{figure:pathfinder}, PC, GES and DAGMA do not terminate within the time limit of 30
 minutes. Here, $\text{FLOP}_0$ and BOSS yield roughly similar SHD. However,
 with an increasing number of ILS iterations, the SHD gets worse for
 $\text{FLOP}_{20}$ and $\text{FLOP}_{100}$. To analyze this behaviour further,
 we show the BIC score difference to the ground-truth DAG on the right of
 Figure~\ref{figure:pathfinder}. Indeed, all reported methods yield better
 BIC scores than the true DAG and ILS does find even better-scoring graphs,
 which, in this case, are further from the ground truth. Again, faithfulness
 violations promoted by the underlying graph structure may be the issue here,
 even though closer investigations are needed.
 
 \subsection{Non-linear data} \label{subsection:nonlinear}
 As settings where the linear Gaussian BIC is misspecified, we consider
 non-linear data generated from a randomly initialized multi-layer perceptron
 (MLP) with a single hidden layer of size 100 and sigmoid activation, as
 described in Appendix~C.2.2 in~\citep{bello2022dagma} and from sampled Gaussian
 process regressions with a unit bandwidth RBF kernel as proposed
 in~\citep{rolland22a}. In both settings, the ground-truth DAG is generated by
 orienting an ER graph with 25 nodes and average degree 4, thus containing on
 average 50 edges, according to a linear order that is drawn uniformly at
 random. We consider the same algorithms as before
 with the same parameter choice and score. They are hence not tuned towards the
 non-linear setting. Additionally, we include the non-linear DAGMA algorithm
 from~\cite{bello2022dagma}. In Figure~\ref{figure:nonlinear}, we plot the SHD
 of each method contrasted with the BIC difference to the optimal BIC score
 (for $\lambda_\text{BIC} = 2$) for each of the algorithms (cases where PC does
 not return a valid CPDAG are omitted). As can be seen, the BIC optimum does
 not correspond to low SHD in both settings with GES, BOSS and
 $\text{FLOP}_{100}$ having similar performance, and the ground truth having
 suboptimal BIC scores. For the MLP setting, the non-linear version of DAGMA is
 the best method for graph recovery, however, in the GP setting it is not
 better than the other approaches. It is also the by far slowest method, taking
 over 5 minutes per instance. The assumptions of the LiNGAM algorithm are also
 violated by the non-linearities and it is clearly the worst-performing
 algorithm among the presented ones.
 
 \begin{figure}
   \begin{tikzpicture}[scale=0.54]
     \input{img/nonlinear_mlp.tikz}
     \node[font=\small] at (0.650928969005743869, 1.20058963282937365) {$\text{FLOP}_{100}$};
     \node[font=\small] at (0.86520515335865, 1.805) {BOSS};
     \node[font=\small] at (0.20520515335865, 2.65) {GES};
     \node[font=\small] at (3.15520515335865, 2.750) {PC};
     \node[font=\small] at (3.402739895528542, 0.6002937365010799) {$\text{DAGMA}_{\text{non-linear}}$};
     \node[font=\small] at (3.602739895528542, 1.48002937365010799) {$\text{DAGMA}_{\text{linear}}$};
     \node[font=\small] at (2.379726785980697, 4.7) {LiNGAM};
     \node[font=\small] at (2.179726785980697, 0.0) {true};
   \end{tikzpicture}
   \begin{tikzpicture}[scale=0.54]
     \input{img/nonlinear_gp.tikz}
     \node[font=\small] at (0.650928969005743869, 3.73058963282937365) {$\text{FLOP}_{100}$};
     \node[font=\small] at (0.95520515335865, 4.405) {BOSS};
     \node[font=\small] at (0.20520515335865, 5.25) {GES};
     \node[font=\small] at (1.85520515335865, 4.950) {PC};
     \node[font=\small] at (5.582739895528542, 4.5802937365010799) {$\text{DAGMA}_{\text{non-linear}}$};
     \node[font=\small] at (7.402739895528542, 5.18002937365010799) {$\text{DAGMA}_{\text{linear}}$};
     \node[font=\small] at (3.379726785980697, 6.2) {LiNGAM};
     \node[font=\small] at (5.879726785980697, 0.0) {true};
   \end{tikzpicture}
   \caption{BIC score difference to the BIC optimum plotted against SHD for non-linear data generated based on MLPs with a single hidden layer on the left. On the right, non-linearities are generated from sampled Gaussian process regressions with a unit bandwidth RBF kernel. Both settings use ER graphs with 25 nodes and average degree 4, thus the ground truth contains 50 edges on average. In the MLP setting, $\text{FLOP}_{100}$ finds the BIC optimum in 44\% of cases, BOSS finds it in 12\% of cases and GES in 4\% of cases. In the GP setting, $\text{FLOP}_{100}$ finds the BIC optimum in 72\% of cases, BOSS in 16\% of cases and GES in 30\% of cases.}
   \label{figure:nonlinear}
 \end{figure}

 \subsection{causalAssembly data set} \label{subsection:causalAssembly}
 We show the results on the causalAssembly dataset introduced by~\cite{gobler2024texttt} in Figure~\ref{figure:causalAssembly:shdbic}. The ground-truth DAG consists of 98 nodes and 485 edges. We subsample 5000 observations with replacement from the data set 50 times and run the algorithms on this subsampled data. We exclude DAGMA and LiNGAM from the plots as they yield significantly larger SHD, which lies above 550, and take much longer than the competing algorithms, namely more than 30 seconds in the case of DAGMA and more than 200 seconds in the case of LiNGAM. The remaining algorithms return results of similar quality, with notable improvements through the ILS restarts that FLOP uses. These small improvements stem from better BIC score optimization as shown in the right plot. Here, the BIC difference to the true graph is reported and it is clear that all methods return graphs with much better BIC scores than that of the ground truth, suggesting score misspecification. 
 
 \begin{figure}
   \begin{tikzpicture}[scale=0.54]
     \input{img/causalAssembly_shd.tikz}
     \node[font=\small] at (0.450928969005743869, 5.40058963282937365) {$\text{FLOP}_0$};
     \node[font=\small] at (2.379726785980697, 5.1) {$\text{FLOP}_{20}$};
     \node[font=\small] at (9.152739895528542, 5.45002937365010799) {$\text{FLOP}_{100}$};
     \node[font=\small] at (5.50520515335865, 5.60) {BOSS};
     \node[font=\small] at (4.50520515335865, 6.90) {GES};
     \node[font=\small] at (7.10520515335865, 6.550) {PC};
   \end{tikzpicture}
   \begin{tikzpicture}[scale=0.54]
     \input{img/causalAssembly_bic.tikz}
     \node[font=\small] at (0.450928969005743869, -6.20058963282937365) {$\text{FLOP}_0$};
     \node[font=\small] at (2.529726785980697, -6.40) {$\text{FLOP}_{20}$};
     \node[font=\small] at (8.552739895528542, -6.5002937365010799) {$\text{FLOP}_{100}$};
     \node[font=\small] at (6.00520515335865, -5.30) {GES};
     \node[font=\small] at (4.650520515335865, -7.22) {BOSS};
   \end{tikzpicture}
   \caption{Run-time against SHD for the causalAssembly data on the left. On the right, the BIC score difference to the BIC of the ground truth is reported.  LiNGAM and DAGMA are not included because both are considerably slower on these instances (with DAGMA needing more than 30 seconds per instance and LiNGAM more than 200 seconds) and obtain significantly worse SHD compared to the other methods, typically above 550 for both algorithms. As a point-of-reference, the ground-truth graphs consists of 485 edges, thus these two methods give worse SHDs than the empty graph. All methods optimizing the BIC, shown in the right graph, yield BIC scores clearly lower than that of the true graph, indicating score misspecification for the linear Gaussian BIC.}
   \label{figure:causalAssembly:shdbic}
 \end{figure}
 
 \begin{figure}
   \begin{tikzpicture}[scale=0.54]
     \input{img/default_lambda.tikz}
   \end{tikzpicture}
   \begin{tikzpicture}[scale=0.54]
     \input{img/default_perturb.tikz}
   \end{tikzpicture}
   \caption{SHD for $\text{FLOP}_{20}$ with different choices of $\lambda_{\text{BIC}}$ on the left (ER graphs with 50 nodes and average degree 8) and with different factors $x$ controlling the number of swaps in a perturbation on the right (ER graphs with 50 nodes and average degree 16). For the $\lambda_{\text{BIC}}$, as explained by the derivation of the extended BIC~\citep{foygel2010extended}, values higher than $1$ are needed on finite samples with $2$ being a common choice. For the perturbations, it can be seen that many choices for the number of random swaps are effective (the exception being no perturbations, and thus no ILS at all, which is shown at $x=0$, yielding an SHD in the hundreds for this setting), with outliers increasing for too few or too many swaps.}
   \label{figure:hyperparameters}
 \end{figure}
 
 \subsection{Sachs data set} \label{subsection:sachs}
 We also evaluate the algorithms on the Sachs dataset~\citep{sachs2005causal}, which consists of 11 nodes, and compute the SHD with regard to the CPDAG of the ground truth consisting of 17 edges. We run each algorithm using the same hyperparameters as before on 50 bootstrap samples of the 853 observations in the data set. As result, we observed $\text{FLOP}_{20}$, BOSS and GES performing on par, all yielding an average SHD of $12.58$. The other algorithms yield similar results, with DAGMA having the best performance with an average SHD of $11.7$. The PC algorithm obtains an average SHD of $12.56$ and LiNGAM an average SHD of $14.18$. 
 
 \subsection{Results for different parameters choices}
 FLOP has two parameters that need to be chosen by the user. First, $\lambda_{\text{BIC}}$ scales the penalty term of the BIC and, second, the number of ILS restarts control the amount of compute that is invested. For the latter parameter, we have typically shown the simulation results for multiple choices, such as $\text{FLOP}_{0}$, when no restarts are performed, as well as $\text{FLOP}_{20}$ and $\text{FLOP}_{100}$ with 20 and a 100 restarts, respectively. We also note that more ILS iterations can only improve the BIC optimization. 
 
 For $\lambda_{\text{BIC}}$ on the other hand, we choose the value $2$, which is the standard setting from the literature. As~\cite{foygel2010extended} have shown, a larger value than $1$ should be chosen to recover the structure of graphical models, while any constant value guarantees asymptotic consistency. The results on the left of Figure~\ref{figure:hyperparameters} confirm this, showing that for $\lambda_{\text{BIC}}$ larger or equal to $2$, graphs close to the ground truth are recovered by $\text{FLOP}_{20}$, while smaller choices of $\lambda_{\text{BIC}}$ yields spurious edges and thus a higher SHD.
 
 Finally, for the ILS perturbations, FLOP defaults to $\ln p$ many random swaps. The software interface of FLOP does not allow tuning this parameter, as we found it to be a stable default choice. This is confirmed on the right of Figure~\ref{figure:hyperparameters}, which runs FLOP with $x \cdot \ln p$ many random swaps for $x \in \{0, 1/4, 1/2, 3/4, 1, 4/3, 2, 4\}$. In the case that $x$ is set to zero, which corresponds to not running ILS, this yields an SHD that is often in the hundreds. Conversely, any of the positive choices of $x$ lead to good performance. The best results are obtained for $x$ between $3/4$ and $2$, while for the largest and smallest values of $x$ the number of outliers increases.
 
 \subsection{Ancestor Adjustment Distance}
 We show the Ancestor Adjustment Identification Distance (AID) as another metric
 for evaluating the learned graphs~\citep{henckeladjustment}. It effectively
 counts the number of mistakes one would make if one used the learned graph to
 select valid adjustment sets (using the ancestors of a node) instead of the
 ground-truth graph. Figure~\ref{figure:aid:first} shows the AIDs for a selection of the previous simulation results. We note that we omit data points where PC does not return a CPDAG (as is well-known to happen on finite samples). For example, on ER graphs with 500 nodes, the PC algorithm does not yield a single valid CPDAG.

\subsection{Other Gradient-Based Algorithms}
Due to the choice of the least-squares loss function, for DAGMA and other popular gradient-based methods, one implicitly assumes perfect varsortability (or even equal noise variances in the underlying linear ANM) to recover the underlying DAG as the unique score-optimal DAG~\citep{peters2014identifiability,reisach2021beware,reisach2023sorting}. Because we compare the methods on \emph{standardized} data in our simulations (with the execption of Subsection~\ref{subsection:rawdata}), this assumption is violated. In comparison, the GOLEM algorithm proposed by~\cite{ng2020role} has a version that is developed specifically for the non-equal noise variance case. We ran this algorithm available in the \texttt{gCastle} package (version 1.0.4) with default parameters on the standard setting in this work (ER graphs with 50 nodes and average degree 8) and observed an average SHD of $176.6$. This is slightly better than DAGMA with an SHD of $195.08$ (which may also be partially due to different hyperparameter choices), but clearly worse than $\text{FLOP}_{100}$ with an average SHD of $1.34$. The GOLEM algorithm also had a run-time of more than 200 seconds per instances, whereas running $\text{FLOP}_{100}$ took around 5 seconds per instances.
 
 \begin{figure}
   \centering
   \begin{tikzpicture}[scale=0.52]
     \input{img/sf_aid.tikz}
     \node[font=\small] at (0.450928969005743869, 3.90058963282937365) {$\text{FLOP}_0$};
     \node[font=\small] at (1.329726785980697, 0.55) {$\text{FLOP}_{20}$};
     \node[font=\small] at (5.052739895528542, 0.4500937365010799) {$\text{FLOP}_{100}$};
     \node[font=\small] at (5.70520515335865, 2.40) {BOSS};
     \node[font=\small] at (2.931505664393822, 4.5746652267818578) {PC};
     \node[font=\small] at (4.409976764476, 6.883952483801296) {GES};
     \node[font=\small] at (8.489976764476, 4.843952483801296) {DAGMA};
   \end{tikzpicture}
   \begin{tikzpicture}[scale=0.52]
     \input{img/onion_aid.tikz}
     \node[font=\small] at (0.450928969005743869, 3.75058963282937365) {$\text{FLOP}_0$};
     \node[font=\small] at (1.259726785980697, 5.05) {$\text{FLOP}_{20}$};
     \node[font=\small] at (2.652739895528542, 3.9002937365010799) {$\text{FLOP}_{100}$};
     \node[font=\small] at (4.60520515335865, 4.40) {BOSS};
     \node[font=\small] at (4.801505664393822, 6.1746652267818578) {PC};
     \node[font=\small] at (2.089976764476, 6.043952483801296) {GES};
     \node[font=\small] at (7.789976764476, 5.043952483801296) {DAGMA};
   \end{tikzpicture}
   \vskip 1em
   \begin{tikzpicture}[scale=0.52]
     \input{img/large_accuracy_250_aid.tikz}
     \node[font=\small] at (0.450928969005743869, 1.10058963282937365) {$\text{FLOP}_0$};
     \node[font=\small] at (2.429726785980697, 0.05) {$\text{FLOP}_{20}$};
     \node[font=\small] at (4.852739895528542, 0.8502937365010799) {$\text{FLOP}_{100}$};
     \node[font=\small] at (7.00520515335865, 1.20) {BOSS};
     \node[font=\small] at (1.001505664393822, 6.0946652267818578) {PC};
     \node[font=\small] at (2.139976764476, 3.843952483801296) {GES};
   \end{tikzpicture}
   \begin{tikzpicture}[scale=0.52]
     \input{img/large_accuracy_500_aid.tikz}
     \node[font=\small] at (0.450928969005743869, 1.20058963282937365) {$\text{FLOP}_0$};
     \node[font=\small] at (3.529726785980697, 0.25) {$\text{FLOP}_{20}$};
     \node[font=\small] at (8.652739895528542, 0.8002937365010799) {$\text{FLOP}_{100}$};
     \node[font=\small] at (2.489976764476, 2.543952483801296) {GES};
   \end{tikzpicture}
   \vskip 1em
   \begin{tikzpicture}[scale=0.52]
     \input{img/alarm_aid.tikz}
     \node[font=\small] at (0.450928969005743869, 3.00058963282937365) {$\text{FLOP}_0$};
     \node[font=\small] at (1.229726785980697, 0.85) {$\text{FLOP}_{20}$};
     \node[font=\small] at (2.552739895528542, -0.1002937365010799) {$\text{FLOP}_{100}$};
     \node[font=\small] at (5.30520515335865, 2.00) {BOSS};
     \node[font=\small] at (4.471505664393822, 1.1746652267818578) {PC};
     \node[font=\small] at (2.989976764476, 2.943952483801296) {GES};
     \node[font=\small] at (8.989976764476, 4.643952483801296) {DAGMA};
   \end{tikzpicture}
   \begin{tikzpicture}[scale=0.52]
     \input{img/pathfinder_aid.tikz}
     \node[font=\small] at (0.450928969005743869, 3.60058963282937365) {$\text{FLOP}_0$};
     \node[font=\small] at (1.729726785980697, 5.7) {$\text{FLOP}_{20}$};
     \node[font=\small] at (6.352739895528542, 5.3002937365010799) {$\text{FLOP}_{100}$};
     \node[font=\small] at (6.40520515335865, 4.00) {BOSS};
   \end{tikzpicture}
   \caption{Run-time and AID for SF graphs (top left), data sampled by the DAG-adaptation of the Onion method (top right), ER graphs with 250 nodes (center left), ER graphs with 500 nodes (center right), the Alarm network (bottom left) and the Pathfinder network (bottom right).}
   \label{figure:aid:first}
 \end{figure}

\end{document}